\documentclass[11pt]{article}

\usepackage{latexsym}
\usepackage[T1]{fontenc}
\usepackage[utf8]{inputenc}
\usepackage[letterpaper,margin=1in]{geometry}
\usepackage{microtype}
\usepackage{authblk}

\usepackage{amsmath,amssymb,amsthm,mathtools}
\usepackage{algpseudocode}
\usepackage{booktabs}
\usepackage{array}
\usepackage{multirow}
\usepackage{longtable}
\usepackage{float}
\usepackage{graphicx}
\usepackage[numbers,square,sort&compress]{natbib}
\usepackage{xcolor}
\usepackage{url}
\usepackage{makecell}

\newcounter{algorithm}
\newenvironment{algorithm}[1][]{%
  \refstepcounter{algorithm}%
  \par\medskip\noindent
  \begin{minipage}{\linewidth}
  \hrule\smallskip
  \def\caption##1{\noindent\textbf{Algorithm~\thealgorithm:} ##1\par\smallskip\hrule\smallskip}%
}{%
  \par\smallskip\hrule
  \end{minipage}
  \par\medskip
}

\graphicspath{{figures/}{../figures/}}

\newtheorem{assumption}{Assumption}

\newtheorem{proposition}{Proposition}

\newtheorem{lemma}{Lemma}

\newcommand{\E}{\mathbb{E}}
\newcommand{\Pbb}{\mathbb{P}}

\title{Conformal Certification of Reasoning Trace Prefixes}

\author[1]{Matt Y. Cheung}
\author[1]{Ashok Veeraraghavan}
\author[2]{Hanjie Chen}
\author[1]{Guha Balakrishnan}
\affil[1]{Department of Electrical \& Computer Engineering, Rice University}
\affil[2]{Department of Computer Science, Rice University}

\begin{document}
\maketitle

\begin{abstract}
Language model reasoning traces are rarely all-or-nothing; they frequently contain valid intermediate steps before a critical error occurs. Existing uncertainty quantification methods typically certify final answers or entire responses, failing to provide statistical guarantees for the proportion of a sequential trace that can be safely retained. To address this, we introduce CROP (Conformal Reasoning Output Prefixes), a verifier-agnostic calibration procedure for clean-prefix certification. Given any step-level risk proxy, CROP selects a calibrated threshold and returns the longest contiguous prefix whose step risk proxies remain below it, routing the uncertified suffix for downstream review or repair. Assuming exchangeability, CROP rigorously controls the marginal probability that the returned prefix contains an annotated error. Across six process-labeled reasoning datasets, we demonstrate that standard step-level metrics such as AUROC do not fully capture prefix utility, suggesting verifiers should instead be evaluated by certified prefix length. Furthermore, CROP balances over- and under-withholding and improves downstream repair accuracy by preserving valid intermediate reasoning while discarding misleading suffixes. 
Ultimately, this work positions prefix certification as a rigorous, practical bridge between process supervision, abstention, and repair.
Code available at \url{https://github.com/matthewyccheung/crop}.
\end{abstract}

\begin{figure}[t]
\centering
\includegraphics[width=0.5\columnwidth]{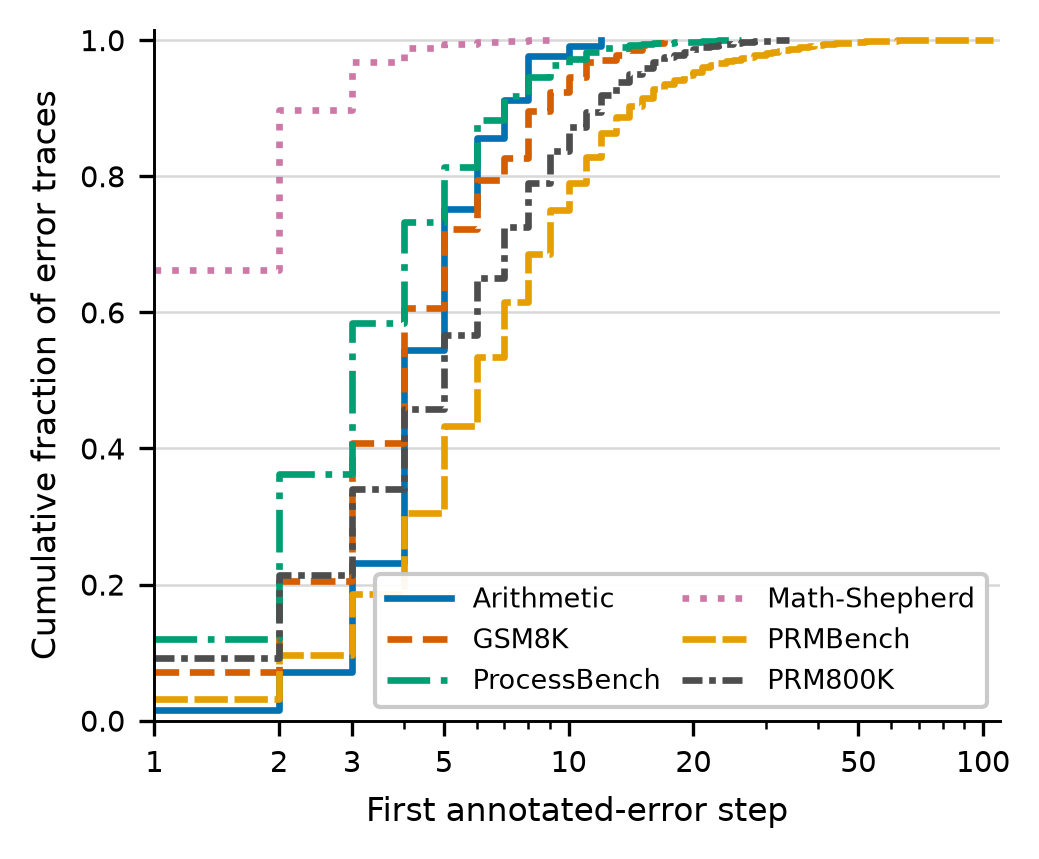}
\caption{\textbf{Erroneous reasoning traces often contain correct and reusable initial steps.} For six diverse reasoning datasets, we plot the empirical CDF of the step index of the first annotated error over all erroneous traces. First errors frequently occur after multiple correct steps, motivating certification and reuse of initial steps rather than treating the entire trace as unusable.}
\label{fig:first_error_step_cdf}
\end{figure}

\section{Introduction}
\label{sec:introduction}
Modern language models (LMs) solve complex problems by generating reasoning traces, or ``chains-of-thought''~\cite{wei2022chain}. When an LM yields an incorrect final answer, its trace is rarely entirely wrong; valid early reasoning is typically followed by a critical mistake due to error propagation or context loss (Figure~\ref{fig:first_error_step_cdf}). 
Traditional answer-level UQ compresses this structure into a scalar
confidence or binary accept/reject decision~\cite{kadavath2022language,kuhn2023semantic,yadkori2024mitigating}, and does not directly identify which steps to trust or route for repair.
Alternatively, step-level UQ signals ranging from likelihood statistics~\cite{lin2023generating} to learned process verifiers~\cite{lightman2024let,park2025know} produce uncalibrated values that lack formal statistical guarantees when deployed.

\begin{figure*}[t]
\centering
\includegraphics[width=\linewidth]{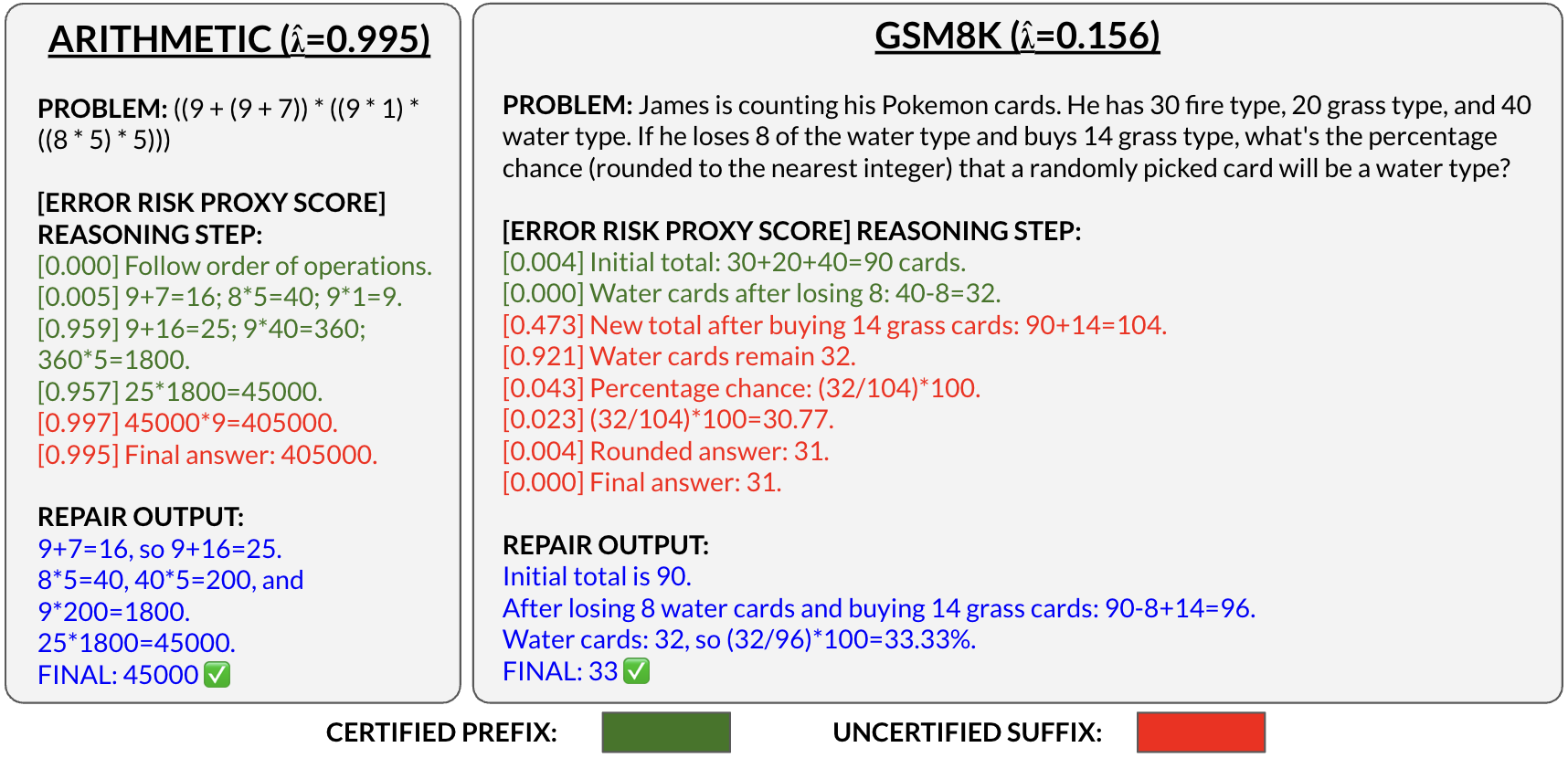}
\caption{\textbf{CROP returns a calibrated prefix of a completed reasoning
trace.} We show 2 example traces solving problems from the Arithmetic and GSM8K datasets~\cite{zhao2025verifying}. For each step, CROP computes an error risk proxy score (in square brackets), with larger values indicating greater estimated error risk. Using held-out calibration reasoning instances, CROP selects a domain-specific, conformally calibrated score threshold $\widehat\lambda$ that controls the marginal probability that the retained prefix contains an annotated error under exchangeability. At test time, CROP returns the longest initial prefix (green) in which no proxy score exceeds the threshold. The remaining suffix (red) is left uncertified and can be routed to downstream review or repair, whose output is shown in blue.}
\label{fig:crop_overview}
\end{figure*}

Conformal prediction (CP) offers a robust statistical framework for converting arbitrary model-derived scores into rigorous guarantees. 
However, recent adaptations of CP for LMs 
typically certify the final label~\cite{kumar2023conformal}, candidate answer sets~\cite{quach2024conformal}, factual claims~\cite{mohri2024language}, or the entire response~\cite{yadkori2024mitigating}.
These choices do not directly address process-level reasoning: certifying an
entire trace can discard valid early work because of a single late mistake,
while certifying claims independently does not identify the longest initial
sequence of reasoning that can be retained while controlling annotated-error
contamination.
We refer to this objective as \emph{clean-prefix certification}: selecting a
contiguous prefix from the beginning of a reasoning trace while controlling
the marginal probability that it contains an annotated error.

To bridge this gap, we introduce \emph{Conformal Reasoning Output Prefixes} (CROP), a verifier-agnostic calibration procedure that extracts a contiguous prefix of reasoning steps while controlling the marginal probability that it contains an annotated error (Figure~\ref{fig:crop_overview}). 
Given a reasoning instance $X=(q,\tau)$ with question $q$ and trace $\tau=(\tau_1,\ldots,\tau_T)$, CROP uses a fixed step-level risk proxy $s$ to assign each step $\tau_t$ an error-oriented score $R_t=s_t(X)$. The proxy can be any signal, such as a process reward model, learned detector, or likelihood statistic, with larger scores indicating higher estimated error risk.
Using held-out labeled reasoning instances and a target risk level $\alpha$, CROP returns the longest contiguous prefix $\tau_{1:m}$ whose proxy scores remain below a calibrated threshold. One may truncate and route the uncertified suffix $\tau_{m+1:T}$ for downstream review or repair, for example, using backtracking or self-correction techniques (see Section~\ref{sec:related_work}).


We evaluated CROP across six process-labeled reasoning datasets: Arithmetic~\cite{zhao2025verifying}, GSM8K~\cite{zhao2025verifying}, ProcessBench~\cite{zheng2025processbench}, Math-Shepherd~\cite{wang2024math}, PRMBench~\cite{song2025prmbench}, and PRM800K~\cite{lightman2024let}. We demonstrate that standard step-level metrics such as AUROC fail to fully capture fixed-risk prefix utility, underscoring the necessity of our prefix-calibrated approach. On Arithmetic and GSM8K, CROP selects retained prefixes closer to the annotated clean prefixes than non-prefix interfaces. By preserving a calibrated prefix rather than accepting or discarding the whole trace, CROP better balances between over-withholding (discarding valid steps) and unsafe overshoot (retaining unsafe steps), ultimately improving downstream repair accuracy over standard baselines in the majority of our benchmarks.

We summarize our contributions as follows: \textbf{1. Formalizing Clean-Prefix Certification:} We define clean-prefix certification for sequential reasoning traces and introduce CROP as a verifier-agnostic calibration layer that converts step-level risk proxies into a risk-controlled contiguous prefix. \textbf{2. Theoretical Guarantees:} We prove finite-sample marginal reasoning instance-level control of annotated prefix contamination under the assumption of exchangeability. \textbf{3. Empirical Effectiveness:} We demonstrate that AUROC is an insufficient predictor of prefix utility. CROP selects stopping points that strikes a better balance between over- and unsafe overshoot, preserving valid intermediate reasoning and enhancing downstream repair accuracy over uncalibrated UQ baselines on six well-established and diverse benchmarks. 
\section{Related Work}
\label{sec:related_work}
\paragraph{Uncertainty Quantification in LMs.}
Modern language models expose various uncertainty signals, but their raw values lack inherent calibration. Prior UQ methods leverage likelihoods~\cite{kadavath2022language,lin2023generating}, verbalized confidence~\cite{lin2022teaching,tian2023just}, self-consistency~\cite{wang2022self}, semantic uncertainty~\cite{kuhn2023semantic}, sampling disagreement~\cite{manakul2023selfcheckgpt}, and hidden-state features~\cite{azaria2023internal}. While these signals rank relative risk, they require additional calibration layers to provide finite-sample reliability guarantees.

\paragraph{Reasoning Verifiers and Process Supervision.}
Beyond general generation, specialized verifiers evaluate multi-step reasoning. These include process reward models (PRMs)~\cite{lightman2024let,prmlessons,yang2024qwen2}, scalable process supervision~\cite{luo2024improve}, attribution graphs~\cite{zhao2025verifying}, and dependency-based step analysis~\cite{mukherjee2025premise}. CROP is verifier-agnostic, natively accepting risk proxies from any of these methods to determine a certified trust boundary.

\paragraph{Backtracking and Self-Correction.}
Recent work leverages backtracking to improve generation, reasoning, and safety. Methods either train models to emit explicit backtracking tokens or apply inference-time search to revise earlier decisions~\citep{yang2025step, qin2025backtrack, cai2025much, gandhi2025cognitive}, or undo problematic continuations to recover from unsafe generations~\citep{zhang2025backtracking, sel2025backtracking}. These approaches modify the generation or repair policy itself. CROP is complementary to these methods: while CROP does not prescribe \emph{how} a suffix should be regenerated, it provides backtracking and self-correction systems with a mathematically guaranteed, calibrated restart point.

\paragraph{Conformal Prediction in LMs.} Conformal prediction (CP) and conformal risk control (CRC) convert arbitrary scores into prediction rules with marginal guarantees under exchangeability~\cite{vovk2005algorithmic,fontana2023conformal,shafer2008tutorial,angelopoulos2023conformal,angelopoulos2024theoretical,angelopoulos2024conformal}. Recent adaptations for LM UQ~\cite{campos2024conformal} span multiple-choice QA~\cite{kumar2023conformal,vishwakarma2025prune}, open-ended generation~\cite{quach2024conformal,ulmer2024non,wang2024conu,kladny2025conformal}, and whole-trace abstention~\cite{machcha2025large,yadkori2024mitigating}. 
Sequential methods~\cite{quach2024conformal,kladny2025conformal}
construct or filter sets of candidate generations rather than steps within
a single reasoning trace.
Closest to our problem are conformal factuality methods~\citep{mohri2024language,rubin2025conformal,cherian2024large}. However, these prior works typically certify final labels, answer sets, individual factual claims, dependency structures, or all-or-nothing abstention decisions. CROP diverges by certifying a contiguous trace prefix, preserving the original step order to maximize reusable context for downstream repair.
\section{Problem Setup}
\label{sec:problem_setup}

\paragraph{Setup.} Let $q$ denote a question and $\tau=(\tau_1,\ldots,\tau_T)$ the corresponding generated reasoning trace, with step $\tau_t$ at index $t$. We denote the completed reasoning instance $X=(q,\tau)$. A reasoning trace prefix is the contiguous sequence $\tau_{1:m}=(\tau_1,\ldots,\tau_m)$, leaving the uncertified suffix $\tau_{m+1:T}=(\tau_{m+1},\ldots,\tau_T)$. Dataset annotation assigns each step a binary error label $Y=(Y_1,\ldots,Y_T) \in \{0,1\}^T$, where $Y_t=1$ indicates an annotated error. This protocol (e.g., human labeling~\cite{lightman2024let,zheng2025processbench}, programmatic verification~\cite{wang2024math,zhao2025verifying}, or LLM-as-a-judge~\cite{gu2026survey,zheng2023judging,zhao2025verifying}) must be fixed before calibration. 
CROP controls risk relative to these annotations. We denote a labeled reasoning instance by $Z=(X,Y)$.

\paragraph{Step-Level Risk Proxy.}
For a reasoning instance $X$ containing $T$ steps, a vector-valued step-level risk proxy function $s$ maps $X$ to $s(X)=\bigl(s_1(X),\ldots,s_T(X)\bigr)\in\mathbb{R}^T$, assigning one error-oriented score to each step. 
We write $R_{1:T}=s(X)$ and $R_t=s_t(X)$, where larger values indicate stronger evidence of an annotated error at step $\tau_t$. Examples include process reward models~\citep{lightman2024let,zheng2025processbench}, likelihoods~\citep{kadavath2022language,lin2023generating}, hidden-state statistics~\citep{azaria2023internal}, and hand-designed heuristics.

The proxy function need not be calibrated or perfectly accurate, but it must be fixed before conformal calibration. 
At the reasoning-instance level, labeled instances have three disjoint roles: proxy-training labels fit $s$ when applicable, calibration labels select $\widehat\lambda$, and test labels are used only for retrospective evaluation. 
At deployment, CROP computes $s(X)$ from the unlabeled instance $X$ and requires no step-level error labels.

\paragraph{Exchangeability.} 
We assume the standard split CP exchangeability condition on the data and choices~\citep{vovk2005algorithmic,shafer2008tutorial,angelopoulos2023conformal}.
\begin{assumption}[Reasoning-Instance Exchangeability]
\label{assump:exchangeability}
Conditioning on all data used to construct the step-level risk proxy function and all preprocessing and calibration-design choices fixed before observing the calibration instances, the labeled calibration reasoning instances $Z_1,\ldots,Z_n$ and future labeled reasoning instance $Z_{n+1}$ are exchangeable.
\end{assumption}
Assumption~\ref{assump:exchangeability} enforces reasoning-instance-level exchangeability after proxy training, preprocessing, and calibration-design choices have been fixed, but before the calibration instances are observed. Steps within a trace may be arbitrarily dependent. The assumption holds, for example, when calibration and future instances are sampled i.i.d. from the same task distribution, generated using the same fixed model and prompting procedure, and annotated under the same fixed protocol.
\section{Method}
\label{sec:method}
This section introduces the CROP theory and algorithm. Our key observation is that \emph{prefix contamination} -- the event where a retained prefix contains at least one annotated erroneous step -- can be represented as a binary loss that is monotonic with respect to the threshold. Because raising the threshold can only increase the number of retained steps, this loss can transition from non-failure to failure but not in the reverse direction. This monotonicity enables the use of conformal risk control (CRC)~\cite{angelopoulos2024conformal}. Thus, clean-prefix certification reduces to finding the largest threshold whose corrected empirical contamination rate remains below the target risk $\alpha$.

\subsection{Risk Control Primitive}
We first state a generic conformal risk-control lemma for monotone binary losses (Lemma~\ref{lem:monotone_crc}), which is a specialized consequence of CRC~\cite{angelopoulos2024conformal} that CROP leverages to bound prefix contamination.

\begin{lemma}[Monotone Binary-Loss Risk Control]
\label{lem:monotone_crc}
Let $W_1,\ldots,W_n,W_{n+1}$ be exchangeable samples, and let $\{L_\lambda:\lambda\in\Lambda\}$ be a fixed family of binary losses indexed by a finite ordered threshold set $\Lambda$. Assume that $L_\lambda(w)\in\{0,1\}$ is nondecreasing in $\lambda$ for every $w$. For a target risk $\alpha\in(0,1)$, define
\begin{equation}
  \widehat\lambda
  =
  \max\left\{\lambda\in\Lambda:
  \frac{1+\sum_{i=1}^n L_\lambda(W_i)}{n+1}\le \alpha
  \right\},
  \label{eq:generic_crc_rule}
\end{equation}
whenever the feasible set is nonempty. If the feasible set is empty, set
$\widehat\lambda=\bot$, where $L_{\bot}(w)=0$ for every $w$.
Then
\begin{equation}
\E_{W_{1:n+1}}
\left[L_{\widehat\lambda}(W_{n+1})\right]
\le \alpha,
\label{eq:generic_crc_guarantee}
\end{equation}
where the expectation is taken jointly over the exchangeable calibration
sample $W_{1:n}$ and future example $W_{n+1}$. In particular,
$\widehat\lambda$ is random through its dependence on $W_{1:n}$.
\end{lemma}

\begin{proof}[Proof sketch]
Because $L_\lambda(w)$ is binary and nondecreasing in $\lambda$, each sample has a monotonic threshold path, transitioning from $0$ to $1$ at most once. Equation~\eqref{eq:generic_crc_rule} selects the largest threshold where at most $\lfloor \alpha(n+1)\rfloor-1$ calibration samples fail. If the future sample also fails here, it belongs to a pool of at most $\lfloor \alpha(n+1)\rfloor$ failing samples (calibration plus future). By exchangeability, the probability of the future sample being in this failing pool is at most $\lfloor \alpha(n+1)\rfloor/(n+1)\le \alpha$. If the fallback is used, its loss is exactly zero. See Appendix~\ref{app:proofs} for the full proof.
\end{proof}
Lemma~\ref{lem:monotone_crc} proves we can bound the probability of any threshold-monotonic failure event without requiring perfectly calibrated risk proxies or an infallible verifier.

\subsection{CROP Loss and Calibration Rule}
We now instantiate Lemma~\ref{lem:monotone_crc} for reasoning prefixes using the notation from Section~\ref{sec:problem_setup}. For a threshold $\lambda\in\Lambda$, the retained prefix length is:
\begin{equation}
\begin{aligned}
  M_\lambda(R_{1:T})
  &=
  \max\{m\in\{0,\ldots,T\}: \\
  &\qquad R_t\le \lambda \text{ for every } t\le m\}.
\end{aligned}
\label{eq:prefix_length}
\end{equation}
Here, $M_\lambda(R_{1:T})$ is the number of initial reasoning steps retained
at threshold $\lambda$. 
When the score vector is clear from context, we write $M_\lambda=M_\lambda(R_{1:T})$. The retained prefix is $\tau_{1:M_\lambda}$, and the uncertified suffix is $\tau_{M_\lambda+1:T}$.

The CROP failure event (prefix contamination) triggers if the retained prefix includes an annotated error, yielding the binary loss function:
\begin{equation}
  L_\lambda(Z)=\mathbf 1\{\exists t\le M_\lambda(R_{1:T}):Y_t=1\}.
  \label{eq:prefix_loss}
\end{equation}
This loss is monotonic in $\lambda$: increasing $\lambda$ can only increase $M_\lambda$, and once an annotated erroneous step enters the prefix, it remains there for all larger thresholds.

Given calibration instances $Z_1,\ldots,Z_n$, CROP defines the feasible set of thresholds where the empirical contamination rate (with a conformal $+1$ correction) satisfies the target risk $\alpha$:
\begin{equation}
  \mathcal F=
  \left\{\lambda\in\Lambda:
  \frac{1+\sum_{i=1}^n L_\lambda(Z_i)}{n+1}\le \alpha
  \right\}.
  \label{eq:feasible_set}
\end{equation}
If $\mathcal F$ is nonempty, set $\widehat\lambda=\max\mathcal F$; otherwise, set $\widehat\lambda=\bot$, where $M_{\bot}(R_{1:T})=0$. 
CROP's main guarantee is justified by the following proposition, which is a direct application of Lemma~\ref{lem:monotone_crc} to this loss.

\begin{proposition}[Clean-Prefix Risk Control]
\label{prop:crop}
Under Assumption~\ref{assump:exchangeability}, let $\mathcal F$ be defined by
Equation~\eqref{eq:feasible_set}. Set $\widehat\lambda=\max\mathcal F$ when
$\mathcal F$ is nonempty; otherwise, set $\widehat\lambda=\bot$, with
$M_{\bot}(R_{1:T})=0$. For the future labeled reasoning instance
$Z_{n+1}=(X_{n+1},Y_{n+1})$, evaluate $M_{\widehat\lambda}$ on the score
vector $s(X_{n+1})$. The returned prefix satisfies
\begin{equation}
  \Pbb_{Z_{1:n+1}}
  \left\{
    \exists t\le M_{\widehat\lambda}:Y_{n+1,t}=1
  \right\}
  \le \alpha.
  \label{eq:main_guarantee}
\end{equation}
The probability is joint over calibration instances $Z_{1:n}$ and the
future exchangeable instance $Z_{n+1}$.
\end{proposition}
\begin{proof}
For each trace, the retained prefix length $M_\lambda$ is nondecreasing in $\lambda$.
Therefore, the event that the retained prefix contains an annotated error is also nondecreasing in $\lambda$.
Thus the prefix-contamination loss in Equation~\eqref{eq:prefix_loss} is a monotone binary loss.
Applying Lemma~\ref{lem:monotone_crc} with $W_i=Z_i$ gives
\[
\E_{Z_{1:n+1}}
\left[L_{\widehat\lambda}(Z_{n+1})\right]
\le \alpha,
\]
where the expectation is joint over the calibration instances and the future
exchangeable instance.
Since this binary loss is the indicator of the event in Equation~\eqref{eq:main_guarantee}, its expectation is exactly the probability of prefix contamination.
\end{proof}

Proposition~\ref{prop:crop} establishes finite-sample marginal risk control: over the random calibration sample and a future exchangeable instance, the probability that CROP returns a prefix containing an annotated error is at most $\alpha$. This is not a conditional guarantee for each realized prefix.

\subsection{CROP Algorithm}
Algorithm~\ref{alg:crop} summarizes the resulting split-conformal procedure: compute prefix-contamination losses on calibration traces, select the largest feasible threshold, and apply it to the future reasoning instance.

\begin{algorithm}[t]
\caption{CROP}
\label{alg:crop}
\begin{algorithmic}[1]
\Require
\parbox[t]{0.86\linewidth}{
Calibration instances $\{Z_i=(X_i,Y_i)\}_{i=1}^n$ with
$X_i=(q_i,\tau_i)$; fixed step-level risk proxy function $s$;
ordered threshold set $\Lambda$; target risk $\alpha$; and test
instance $X_{n+1}=(q_{n+1},\tau_{n+1})$
}
\For{each calibration trace $i=1,\ldots,n$}
    \State $R_{i,1:T_i}\gets s(X_i)$
    \For{each threshold $\lambda\in\Lambda$}
        \State $\begin{aligned}
        M_{i,\lambda}\gets \max\bigl\{\,
            &m\in\{0,\ldots,T_i\}:\\[-2pt]
            &R_{i,t}\le\lambda,\quad \forall t\le m
        \,\bigr\}
        \end{aligned}$
        \State $L_{i,\lambda}\gets
        \mathbf{1}\!\left\{
            \exists t\le M_{i,\lambda}:Y_{i,t}=1
        \right\}$
    \EndFor
\EndFor
\State $\displaystyle
\mathcal F\gets
\left\{
\lambda\in\Lambda:
\frac{1+\sum_{i=1}^n L_{i,\lambda}}{n+1}\le\alpha
\right\}$
\If{$\mathcal F=\emptyset$}
    \State \Return
    \parbox[t]{0.70\linewidth}{
    empty-prefix fallback with zero\\
    contamination loss
    }
\EndIf
\State $\widehat\lambda\gets\max\mathcal F$
\State $R_{n+1,1:T_{n+1}}\gets s(X_{n+1})$
\State $\begin{aligned}
\widehat M\gets \max\bigl\{\,
    &m\in\{0,\ldots,T_{n+1}\}:\\[-2pt]
    &R_{n+1,t}\le\widehat\lambda,\quad \forall t\le m
\,\bigr\}
\end{aligned}$
\State \Return
\parbox[t]{0.70\linewidth}{
retained prefix $\tau_{n+1,1:\widehat M}$
}
\end{algorithmic}
\end{algorithm}

\begin{figure*}[t]
\centering
\includegraphics[width=\linewidth]{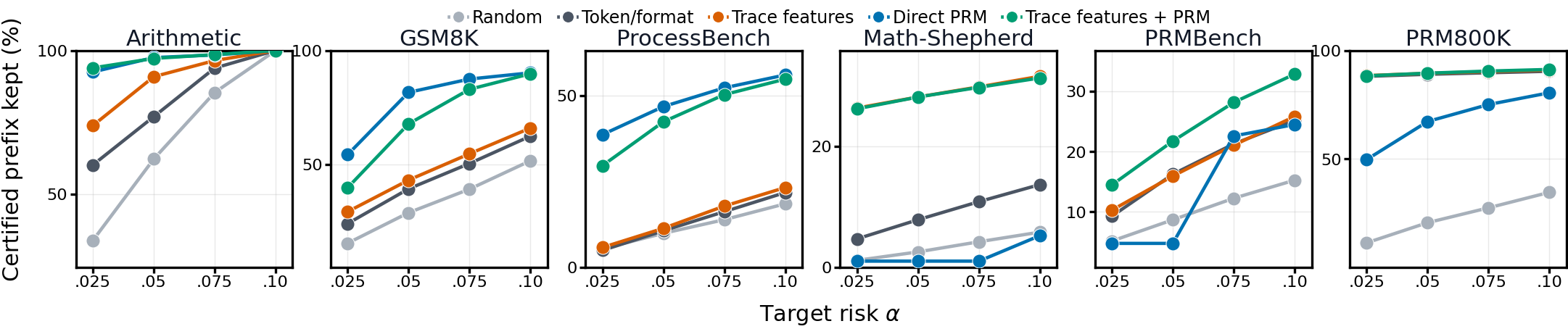}
\caption{\textbf{CROP reveals how efficiently a risk proxy function converts risk budget into certified reasoning.}
We swept the target prefix-contamination risk \(\alpha\) over 10 random splits and reported the mean certified prefix retained after CROP calibration.
The slope of each curve shows how much additional reasoning becomes reusable as the allowed risk is relaxed.
PRM-backed risk proxy functions often retain long prefixes even under strict risk targets, while flat curves show cases where allowing more risk still does not recover much additional reasoning.}
\label{fig:alpha_prefix_curve}
\end{figure*}

\begin{figure*}[t]
\centering
\includegraphics[width=\linewidth]{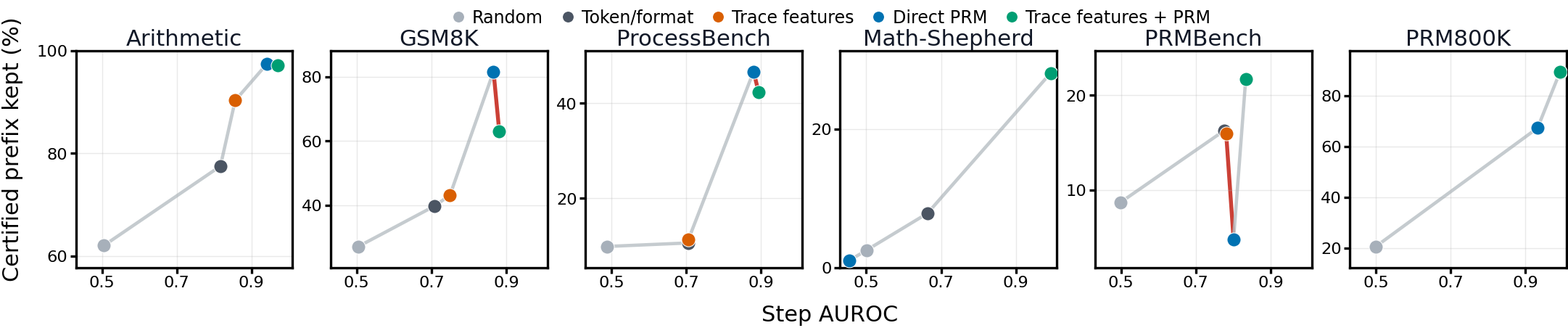}
\caption{\textbf{Step AUROC is an incomplete indicator of fixed-risk prefix utility.}
Each panel shows one dataset at $\alpha=0.05$ over 10 random splits, with each point corresponding to one risk proxy type, averaged over splits; the x-axis is step-level AUROC and the y-axis is retained-prefix fraction. Gray segments mark cases where higher AUROC coincides with a larger retained prefix, while red segments mark inversions where higher AUROC yields a smaller retained prefix. This non-monotonicity shows that verifier selection for CROP should be based on calibrated prefix utility, not step AUROC alone.}
\label{fig:auroc_prefix_utility_main}
\end{figure*}

Algorithm~\ref{alg:crop} is verifier-agnostic and accepts any fixed step-level risk proxy function $s$, which may be based on a process reward model, learned detector, likelihood statistic, or hand-designed heuristic. The conformal guarantee requires reasoning-instance-level exchangeability and that $s$ be fixed before calibration, but it does not require $s$ to be intrinsically accurate or calibrated. More informative risk proxy functions can improve efficiency by retaining longer prefixes at the same target risk, while validity follows from the conformal calibration rule.
\section{Experiments}
\label{sec:experiment}
We evaluated CROP across several step-level risk proxy functions, focusing on three questions: (1) Which risk proxy functions retain the longest prefixes at a fixed contamination-risk target? (2) How closely does the retained prefix approximate the oracle annotated-clean boundary? (3) Does supplying the retained prefix improve downstream repair?

%

\begin{table*}[t]
\centering
\caption{\textbf{CROP improves boundary selection over whole-trace abstention
at matched risk.}
Both methods are separately calibrated at $\alpha=0.05$ using the same
risk proxy sources and trace-level splits. Values are mean percentages
$\pm$ 95\% confidence intervals across 20 splits. Bold marks higher retention
and lower boundary components within each risk-matched pair; empirical risk
is reported as a constraint and is not ranked. Question-only and full-trace
references are excluded because they do not adapt to the requested risk
budget. Results at additional risk targets are reported in Appendix~\ref{app:alpha_sensitivity}.
}
\label{tab:boundary_deviation}
\small
\setlength{\tabcolsep}{4pt}
\begin{tabular}{llrrrrr}
\toprule
Dataset & Method
& \shortstack{Empirical\\risk}
& \shortstack{Steps\\kept}
& \shortstack{Over-\\withholding}
& \shortstack{Unsafe\\overshoot}
& \shortstack{Boundary\\deviation} \\
\midrule
\multirow{2}{*}{Arithmetic}
& Whole-trace abstention
& $4.4\pm0.7$
& $95.8\pm0.7$
& $2.4\pm0.5$
& $2.1\pm0.3$
& $4.5\pm0.2$ \\
& CROP prefix
& $4.3\pm0.4$
& $\mathbf{97.2\pm0.4}$
& $\mathbf{0.6\pm0.2}$
& $\mathbf{1.7\pm0.2}$
& $\mathbf{2.3\pm0.3}$ \\
\midrule
\multirow{2}{*}{GSM8K}
& Whole-trace abstention
& $5.0\pm0.8$
& $77.9\pm2.0$
& $13.4\pm1.8$
& $2.6\pm0.4$
& $15.9\pm1.6$ \\
& CROP prefix
& $5.0\pm0.8$
& $\mathbf{81.6\pm4.3}$
& $\mathbf{9.4\pm4.2}$
& $\mathbf{2.3\pm0.4}$
& $\mathbf{11.7\pm4.0}$ \\
\bottomrule
\end{tabular}
\vspace{2pt}

\begin{minipage}{0.96\linewidth}
\footnotesize
\end{minipage}
\end{table*}

{}{%
    }%
  }%

\paragraph{Risk proxy functions.}
We evaluated five step-level scoring rules (Appendix~\ref{app:details}; scoring costs in Appendix Table~\ref{tab:runtime_cost}):
\begin{enumerate}
\item \emph{Random}: i.i.d.\ uniform step scores, drawn before conformal calibration as an uninformative reference.
\item \emph{Token/format logistic}: logistic regression over surface features such as step position, length, and punctuation.
\item \emph{Trace-feature logistic}: logistic regression combining token/format features with precomputed likelihood, confidence, and hidden-state summaries from Llama-3.1-8B-Instruct.
\item \emph{Direct PRM}: the frozen Qwen2.5-Math-PRM-7B score, with error risk defined as one minus the positive-class probability.
\item \emph{Trace features + PRM}: logistic regression combining the trace-feature inputs with the Direct PRM score. Its output is the predicted probability of the error class $Y_t=1$.
\end{enumerate}
For the three learned logistic models, we used an $\ell_2$ penalty with $C=1.0$ and the L-BFGS optimizer and fitted the model on the proxy-training partition using the local step label $Y_t$. The Direct PRM remained frozen and received no CROP labels.

\paragraph{Evaluation Metrics.}
For a target risk $\alpha=0.05$, we evaluated retained prefixes on held-out test splits. For test trace $i$, let $T_i$ denote its total number of steps and $M_i=M_{\widehat\lambda}(R_{i,1:T_i})$ its retained prefix length. We define the oracle annotated-clean prefix length as
\[
O_i=\min\bigl(\{t-1:Y_{i,t}=1\}\cup\{T_i\}\bigr),
\]
which is the number of steps before the first annotated error, or $T_i$ when no annotated error exists.
We report four metrics:
\begin{enumerate}
\item \emph{Retained-prefix fraction} ($M_i/T_i$): the fraction of original
steps retained.
\item \emph{Prefix-contamination risk}: the fraction of test traces whose
retained prefix contains at least one annotated error.
\item \emph{Boundary deviation} ($|M_i-O_i|/T_i$): the distance from the
oracle boundary, decomposed into over-withholding $(O_i-M_i)_+/T_i$ and unsafe
overshoot $(M_i-O_i)_+/T_i$.
\item \emph{Repair accuracy}: normalized exact-match accuracy of the parsed
final answer against the dataset gold answer.
\end{enumerate}
Prefix-contamination risk evaluates validity relative to the target $\alpha$. Among methods operating at comparable risk, higher retained-prefix fraction and lower boundary deviation indicate greater prefix utility, while higher repair accuracy indicates greater downstream utility.

\paragraph{Datasets.} 
We evaluated CROP on six process-supervision datasets that span various domains and annotation protocols (Appendix Table~\ref{tab:dataset_provenance}). 
The main benchmark comprised 2,819 traces (19,311 steps, 1,006 annotated errors) across Arithmetic and GSM8K~\cite{zhao2025verifying}\footnote{We excluded a densely annotated Boolean domain~\cite{zhao2025verifying}. Only $0.5\%$ of its steps and $1.8\%$ of its traces contain annotated errors. At $\alpha=0.05$, domain-specific calibration retains $100\%$ of the trace with $1.7\%$ empirical prefix-contamination risk, producing a ceiling effect that does not meaningfully distinguish risk-proxy quality. With only approximately five erroneous traces per test split, the domain also cannot support an informative downstream repair comparison.}. 
We also included Math-Shepherd~\cite{wang2024math}, ProcessBench~\cite{zheng2025processbench}, PRMBench~\cite{song2025prmbench}, and PRM800K~\cite{lightman2024let}. 
These supplementary datasets differ in trace construction and annotation semantics and do not provide uniform gold final-answer targets compatible with our repair evaluation. 
To maintain a consistent evaluation setup, we used them for fixed-risk prefix utility and restricted the boundary-quality and repair analyses to Arithmetic and GSM8K. 
All fixed-score evaluations used disjoint reasoning-instance-level 60/20/20 splits: $60\%$ for fitting learned risk proxies, $20\%$ for conformal calibration, and $20\%$ for test-only evaluation. 
For fixed proxies that require no fitting, the calibration and test roles remain unchanged.

  \begin{table*}
\centering
\caption{\textbf{CROP improves downstream repair for the majority of benchmarks.} For Arithmetic and GSM8K and four repair models with $\alpha=0.05$, we show accuracy (\%) over 20 stratified trace-level 60/20/20 repair splits. We compare baselines versus CROP and report CROP$-$best ($\Delta$), which is CROP accuracy minus the best deployable non-CROP accuracy within each split; brackets give split-seed 95\% confidence intervals.}
\label{tab:downstream_repair_usefulness}
\small
\begin{tabular}{llrrrrl}
\toprule
Dataset & Repair model & \shortstack{Question\\only} & \shortstack{Full\\trace} & \shortstack{Whole-trace\\abst.} & \shortstack{CROP\\prefix} & \shortstack{CROP $-$ best\\$\Delta$ [95\% CI]} \\
\midrule
\multirow{4}{*}{Arithmetic} & Gemma 4 E4B IT & 80.5 & 79.9 & 83.4 & \textbf{88.7} & \textbf{+5.32 [4.41, 6.22]} \\
 & Qwen2.5-7B-Instruct & 85.5 & 90.8 & 90.3 & \textbf{92.4} & \textbf{+1.02 [0.28, 1.75]} \\
 & DeepSeek-R1-0528-Qwen3-8B & 63.9 & 78.5 & 78.8 & \textbf{85.0} & \textbf{+5.53 [4.79, 6.28]} \\
 & Llama3.1-8B & 73.7 & 87.3 & \textbf{90.5} & 89.4 & -1.18 [-1.84, -0.53] \\
\addlinespace[2pt]
\multirow{4}{*}{GSM8K} & Gemma 4 E4B IT & 82.3 & \textbf{88.8} & 88.6 & 88.3 & -0.91 [-1.53, -0.29] \\
 & Qwen2.5-7B-Instruct & 80.5 & 85.1 & 86.3 & \textbf{87.6} & \textbf{+0.95 [0.09, 1.81]} \\
 & DeepSeek-R1-0528-Qwen3-8B & 70.0 & 84.9 & 84.6 & \textbf{87.4} & \textbf{+1.83 [0.89, 2.76]} \\
 & Llama3.1-8B & 71.1 & 73.3 & \textbf{74.6} & 73.9 & -1.18 [-2.02, -0.34] \\
\bottomrule
\end{tabular}
\end{table*}
{}{%
    }%
  }%

\subsection{Prefix Utility at Varying Risks}
\label{sec:prefix_utility}
To assess how efficiently each risk proxy converts a risk budget into reusable reasoning, we swept the target risk $\alpha$ and measured the retained-prefix fraction. Figure~\ref{fig:alpha_prefix_curve} illustrates that score geometry influences prefix utility: steep curves indicate that slight risk relaxation yields substantially longer prefixes, whereas flat curves imply the proxy crosses the threshold too early to recover useful reasoning. PRM-backed scores consistently provided the strongest inputs, though the optimal instantiation varied by dataset.

\subsection{Prefix Utility vs. AUROC}
Next, we evaluated fixed-risk prefix utility along with step-level AUROC at $\alpha=0.05$ (Figure~\ref{fig:auroc_prefix_utility_main}). While PRM-backed scores remained dominant, the downward red segments demonstrate that step AUROC is an incomplete predictor of prefix utility. AUROC measures global step ranking, whereas CROP relies on the first threshold crossing in an ordered trace. Exact metric values are detailed in Appendix Table~\ref{tab:target_domain_mondrian}. We isolate the best PRM-backed scores for all subsequent analyses.

\subsection{Boundary Quality of the Retained Prefix}
\label{sec:boundary_quality}
Because retention alone does not distinguish unsafe overshoot from unnecessary withholding of annotated-clean steps, we compared CROP with conformal whole-trace abstention at $\alpha=0.05$ (Table~\ref{tab:boundary_deviation}). 
Based on the preceding proxy comparison, we used PRM-backed proxies with strong fixed-risk prefix utility (Trace+PRM for Arithmetic and Direct PRM for GSM8K), fixed before the boundary and repair analyses without using repair outcomes. Within each dataset, both interfaces used the same proxy and were separately calibrated at the same $\alpha$. 
At comparable empirical risk, CROP retained more steps and reduced over-withholding, unsafe overshoot, and total boundary deviation on both datasets. 
We report the corresponding risk-matched comparison at additional risk targets in Appendix~\ref{app:alpha_sensitivity}.

\subsection{Certified Prefixes for Downstream Repair}
\label{sec:downstream_repair}

Finally, we tested whether the retained prefixes improve downstream final-answer repair across 20 stratified 60/20/20 splits using four models: Gemma 4 E4B IT~\cite{gemma4modelcard}, Qwen2.5-7B-Instruct~\cite{prmlessons,yang2024qwen2}, DeepSeek-R1-0528-Qwen3-8B~\cite{guo2025deepseek}, and Llama3.1-8B-Instruct~\cite{grattafiori2024llama} (Table~\ref{tab:downstream_repair_usefulness}).
We compared CROP against question-only input, full-trace input, and conformal whole-trace abstention.
We report model details in Table~\ref{tab:repair_model_artifacts} and qualitative results in Appendix~\ref{app:qualitative}.

CROP improved over the best deployable non-CROP input in five of eight settings, achieving the largest gains on Arithmetic for DeepSeek and Gemma.
However, results are model-dependent: Llama yielded negative CROP deltas on both domains. One plausible contributor is same-family answer repetition: the original traces were generated by a Llama-family model ~\cite{zhao2025verifying}, and our diagnostic shows that CROP reduced Llama's repetition of the original wrong answer relative to full-trace input but not relative to question-only input (Appendix~\ref{app:repair_anchoring}). We
provide a trace-level example in Table~\ref{tab:llama_repetition_trace_gsm8k}. These results indicate that a CROP prefix can help when removing a misleading suffix outweighs both lost context and contextual anchoring, but downstream gains remain model-specific.
\section{Discussion}
\label{sec:discussion}
We introduced CROP, shifting the focus of uncertainty quantification from final-answer acceptance to post-hoc clean-prefix certification. Given an error-oriented step-level risk proxy, CROP returns a contiguous reasoning prefix with finite-sample marginal risk control under exchangeability, routing the uncertified suffix for downstream review or repair. Our experiments validate three core claims: (1) fixed-risk prefix utility is not reducible to step AUROC; (2) at the primary risk target, CROP reduces boundary deviation relative to risk-matched whole-trace abstention; and (3) CROP improves downstream repair in the majority of the evaluated model–dataset settings. We discuss several further implications below.

\paragraph{PRM gains extend beyond superficial artifacts.}
While simple surface cues (e.g., token and formatting features) provide useful baselines and indicate that the benchmark contains exploitable regularities, Table~\ref{tab:artifact_controls_combined} demonstrates that PRM-backed scores remain substantially stronger even after reweighting examples by these cues and running label/order controls. Thus, these controls indicate that the PRM-backed gains are not explained solely by the evaluated surface and dataset artifacts.

\paragraph{Boundary deviation at common risk budgets.} A fair boundary comparison must separately calibrate each interface to the same target risk. For $\alpha\in\{0.025,0.05,0.10\}$, CROP achieved lower boundary deviation than conformal whole-trace abstention in four of six dataset--risk settings and tied in one. At the primary $\alpha=0.05$ setting, it retained more steps and reduced boundary deviation on both datasets (Appendix~\ref{app:alpha_sensitivity}). We exclude question-only and full-trace inputs from this comparison because they are fixed interfaces rather than risk-calibrated methods: question-only retains no trace context and therefore attains zero contamination trivially, whereas full trace generally exceeds the requested risk target. We retain both elsewhere as downstream utility references.

\paragraph{CROP changes the operational interface.}
Unlike whole-trace abstention, which retains an entire trace or falls back to the question, CROP returns a calibrated prefix and withholds any remaining uncertified suffix. At comparable empirical risk in the pooled Arithmetic and GSM8K evaluation, CROP retained more total step mass with PRM-backed scores despite accepting fewer traces in full, thereby recovering partial reasoning that whole-trace abstention could not retain (Table~\ref{tab:cpcc_vs_abstention}).

\paragraph{Boundary quality governs downstream repair.}
Unlike whole-trace abstention, CROP can preserve useful intermediate reasoning while withholding an uncertified suffix. A retrospective analysis of originally incorrect traces with annotated errors supports this mechanism (Table~\ref{tab:repair_by_boundary}). CROP improved repair when it stopped at or before the annotated error boundary, whereas these gains disappeared when it retained annotated-error steps. Thus, CROP is most useful when reusable initial reasoning is followed by a localized, misleading continuation.

\begin{table}[t]
\centering
\small
\caption{\textbf{Repair gains depend on avoiding annotated-error overshoot.}
Entries are CROP-minus-baseline repair-accuracy differences in percentage points on originally incorrect traces containing an annotated error.
Before, At, and Beyond locate CROP relative to the oracle annotated-clean boundary, defined retrospectively from test labels. 
Values are means $\pm1.96$ standard errors over 20 splits; bold marks positive intervals
excluding zero.}
\label{tab:repair_by_boundary}
\setlength{\tabcolsep}{2.5pt}
\begin{tabular*}{\columnwidth}{
  @{\extracolsep{\fill}}lccc@
  {}}
\toprule
\shortstack{CROP relative\\to oracle} & \shortstack{Question\\only} & \shortstack{Full\\trace} & \shortstack{Whole-trace\\abstentation} \\
\midrule
Before &
$\mathbf{+5.7\pm2.5}$ &
$\mathbf{+21.9\pm3.3}$ &
$\mathbf{+9.0\pm2.9}$ \\
At &
$\mathbf{+5.9\pm2.3}$ &
$\mathbf{+20.3\pm1.8}$ &
$\mathbf{+8.7\pm2.0}$ \\
Beyond &
$-20.1\pm3.0$ &
$-1.4\pm2.3$ &
$-4.3\pm2.4$ \\
\bottomrule
\end{tabular*}
\end{table}

\paragraph{Repair utility depends on the context consumer.}
CROP controls annotated-prefix contamination risk, not the accuracy of an arbitrary downstream repair model. Nevertheless, at $\alpha=0.05$, CROP improved over the best deployable non-CROP input in five of eight model--dataset settings, including five of six settings across Gemma, Qwen, and DeepSeek. The repetition diagnostic in Table~\ref{tab:llama_repair_anchoring} shows that Llama was more likely to reproduce the original wrong answer when supplied with trace context. This behavior is consistent with correlated failure modes between generation and repair~\cite{pmlr-v267-kim25e,tsui2025self,chen2026self}. Thus, downstream utility depends jointly on prefix quality and the repair model's ability to use partial context.

\paragraph{Groupwise diagnostics.}
CROP provides marginal rather than group-conditional risk control, as is
standard for marginal conformal guarantees.
Our held-out diagnostics (Appendix~\ref{app:groupwise_diagnostics}) show that the longest traces and harder problems tend to exhibit higher risk and lower retention, while dataset-level risk remains near the target after in-domain calibration despite variation in retained-prefix utility.
Applications requiring group-specific control can use prespecified, deployment-observable groups for Mondrian~\cite{vovk2003mondrian,vovk2005algorithmic,pmlr-v25-vovk12} or conditional~\cite{gibbs2025conformal,jung2023iclr-batch} calibration, provided that each group has sufficient calibration data.

\paragraph{Empty prefixes.} The calibration-level fallback was never triggered in any runs across datasets, score functions, splits, and $\alpha\in\{0.025,0.05,0.075,0.10\}$. At $\alpha=0.05$, CROP nevertheless returned an empty prefix for $0\%$ of Arithmetic and $6.4\%$ of GSM8K test traces. A nonempty feasible set requires $1/(n+1)\leq\alpha$.

\paragraph{Prefix certification is complementary to conformal factuality.} Conformal factuality methods evaluate final answers, claim sets, or response specificity. CROP instead returns a contiguous initial portion of a reasoning trace with controlled annotated-error risk. These approaches are complementary: claim-level factuality scores, when defined at the reasoning-step level, could serve as risk proxies for CROP, while downstream factuality checks could evaluate a repaired continuation. CROP does not certify semantic truth or the causal faithfulness of written reasoning.

\section{Limitations}

CROP provides finite-sample marginal control of annotated prefix contamination over the calibration sample and a future exchangeable reasoning instance. The reasoning instance is the exchangeable unit, but the result is not a conditional guarantee that every returned prefix is correct with probability $1-\alpha$. It also does not certify final-answer correctness, semantic truth, or causal faithfulness.

The practical interpretation depends on both distributional alignment and annotation quality. Obtaining reliable step-level labels could also be costly. Changes in the task, prompt distribution, generator, risk proxy, or annotation policy can violate exchangeability and require representative recalibration. Moreover, CROP controls agreement with the fixed annotation protocol rather than latent semantic correctness. Noisy, incomplete, or systematically biased annotations weaken the meaning of the certificate.

CROP is also a post-hoc method for completed traces. It identifies a retained prefix and an uncertified suffix, but does not specify how the suffix should be regenerated, repaired, or reviewed. Downstream repair gains therefore depend on the repair model, prompt, and task. Although CROP is verifier-agnostic, its utility depends on risk-proxy quality: weak or poorly ordered risk proxies may satisfy the risk constraint only by returning overly conservative, short prefixes.

\section{Ethical Considerations}
CROP may be misused if its marginal guarantee is overinterpreted as proof that a particular returned prefix is correct or safe without further review. 
In high-stakes settings, CROP should be paired with clear annotation-protocol reporting, distribution-shift monitoring, recalibration when needed, and downstream verification when final answer correctness, factuality, or safety matters.
\section*{Code Availability}
Code to reproduce results can be found at \url{https://github.com/matthewyccheung/crop}.
\section*{Acknowledgments}
MC would like to acknowledge support from a fellowship from the Gulf Coast Consortia on the NLM Training Program in Biomedical Informatics and Data Science T15LM007093.
\bibliographystyle{unsrtnat}
\bibliography{references}
\appendix
\section{Proofs}
\label{app:proofs}

We give the full proof of Lemma~\ref{lem:monotone_crc}.
We condition on all training data used to fit the risk proxy function, as well
as all preprocessing and calibration-design choices fixed before observing
the calibration instances. Under this conditioning, the loss family
$\{L_\lambda:\lambda\in\Lambda\}$ is fixed, and the calibration and future
labeled reasoning instances remain exchangeable by
Assumption~\ref{assump:exchangeability}. For CROP, each exchangeable example
is a complete labeled reasoning instance $W_i=Z_i=(X_i,Y_i)$; steps within an
instance may be arbitrarily dependent. We take $\alpha\in(0,1)$ throughout.

\begin{proof}[Full proof of Lemma~\ref{lem:monotone_crc}]
Because $L_\lambda(w)$ is binary and nondecreasing in $\lambda$, its value can
change from $0$ to $1$ at most once. For each example $W_i$, define its first
failing threshold as
\begin{equation}
  B_i
  =
  \min\{\lambda\in\Lambda:L_\lambda(W_i)=1\},
\end{equation}
with $B_i=\infty$ if no threshold in $\Lambda$ produces a failure. Therefore,
\begin{equation}
  L_\lambda(W_i)=\mathbf 1\{B_i\le\lambda\}.
\end{equation}
Since each $B_i$ is computed from $W_i$ using the same fixed loss family,
exchangeability of $W_{1:n+1}$ implies exchangeability of
$B_{1:n+1}$.

Let
\[
  k=\left\lfloor\alpha(n+1)\right\rfloor
\]
and write the selected threshold explicitly as
$\widehat\lambda=\widehat\lambda(W_{1:n})$. If the feasible set is empty,
the procedure returns $\widehat\lambda=\bot$ with
$L_{\bot}(W_{n+1})=0$, so the result holds immediately. We therefore consider
the case in which a non-fallback threshold is selected.

By the corrected calibration constraint,
\begin{align}
  1+\sum_{i=1}^n L_{\widehat\lambda}(W_i)
  &=
  1+\sum_{i=1}^n
  \mathbf 1\{B_i\le\widehat\lambda\} \\
  &\le \alpha(n+1).
\end{align}
Because the left-hand side is integer-valued,
\begin{equation}
  \sum_{i=1}^n
  \mathbf 1\{B_i\le\widehat\lambda\}
  \le k-1.
  \label{eq:cal_boundaries}
\end{equation}

Now suppose that the future example fails:
\[
  L_{\widehat\lambda}(W_{n+1})=1.
\]
Then $B_{n+1}\le\widehat\lambda$. Together with
Equation~\eqref{eq:cal_boundaries}, this implies
\begin{equation}
  \sum_{i=1}^{n+1}
  \mathbf 1\{B_i\le\widehat\lambda\}
  \le k.
  \label{eq:joint_boundaries}
\end{equation}

To handle ties, let $U_1,\ldots,U_{n+1}$ be independent continuous random
variables, independent of $W_{1:n+1}$, and rank the pairs $(B_i,U_i)$
lexicographically. These pairs are exchangeable and almost surely have
distinct ranks, so the rank of $(B_{n+1},U_{n+1})$ is uniform on
$\{1,\ldots,n+1\}$. On the future failure event,
Equation~\eqref{eq:joint_boundaries} implies that this rank is at most $k$.
Consequently,
\begin{align}
  \E_{W_{1:n+1}}
  \left[L_{\widehat\lambda}(W_{n+1})\right]
  &=
  \Pbb_{W_{1:n+1}} \\
    &\left\{L_{\widehat\lambda}(W_{n+1})=1\right\} \\
  &\le
  \frac{k}{n+1}
  \le \alpha.
\end{align}
This proves the claim.
\end{proof}

\section{Experimental Details}
\label{app:details}

We use these tables to document the empirical setup behind the main claims: what counts as one calibration/test example, how each dataset labels steps, and what it costs to compute each risk proxy.
We use target Arithmetic and GSM8K traces from Circuit-based Reasoning Verification (CRV) process annotations~\cite{zhao2025verifying}.
In CRV, an instruction-tuned generator produces chain-of-thought solutions, and the released annotations mark whether each step is correct.
The Arithmetic labels combine task-specific automatic checks with a stronger-model judge, and GSM8K labels are produced by the CRV stronger-model judging pipeline.
CRV does not report a residual human-audited annotation error rate; its validation protocol instead checks agreement between stronger-model judging and, for synthetic tasks, programmatic state verification.
We use the released annotated traces rather than regenerating model outputs.
For the additional process-supervision datasets, we use their native labels and calibrate CROP separately within each dataset before evaluation.
We include Math-Shepherd~\citep{wang2024math}, ProcessBench~\citep{zheng2025processbench}, PRMBench~\citep{song2025prmbench}, and PRM800K~\citep{lightman2024let}.

We use a CRV target subset with 2819 process-labeled traces, 19311 labeled steps, 1006 annotated step errors, and 402 traces with at least one annotated error.
This consists of 1500 Arithmetic traces with 10316 steps and 349 step errors, plus 1319 GSM8K traces with 8995 steps and 657 step errors.

\paragraph{Splits and threshold grids.}
We make all splits at the reasoning-instance level, so all steps from a trace stay in the same partition.
We use a default 60/20/20 split: 60\% for fitting any learned risk proxy function, 20\% for conformal calibration, and 20\% for test evaluation.
Within each random seed, we use the same train/calibration/test partition for all risk proxy functions; only the fitted risk proxy function changes.
We stratify splits by dataset/domain and by whether the trace contains at least one annotated error.
For the target-domain prefix-utility and repair experiments, we use 20 split seeds, \(2806,\ldots,2825\).
For the additional-dataset fixed-risk experiments, we use 10 split seeds, \(2806,\ldots,2815\).
The boundary-quality and risk-sensitivity analyses use the same 20 repeated target-domain splits.

We run CROP over a finite threshold set \(\Lambda\) fixed before using the calibration labels.
For the target-domain, boundary-quality, alpha-sweep, and repair experiments, we construct a separate grid \(\Lambda_s\) for each risk proxy source \(s\) from the risk proxy-training split: \(\Lambda_s\) contains 201 empirical quantiles of the training-step risk proxies for that source, with repeated values allowed when risk proxies are tied.
This makes the grid data-dependent only through the risk proxy-training instances, not through the calibration labels or test labels.
For the additional-dataset fixed-risk and alpha-sweep experiments, we instead use the fixed numeric grid \(\Lambda=\{0,0.01,\ldots,1.00\}\) with 101 points.
In both protocols, we use conformal calibration only to select the largest feasible threshold from the predeclared grid.

\paragraph{Token/format and trace-feature inventory.}
We include \emph{token/format} as a deliberately cheap risk proxy source.
It uses only surface information that can be read from the text, without looking at model internals, PRM outputs, or labels.
For each step, we record simple counts from the current step and original problem: character length, whitespace-token count, digit count, alphabetic-character count, equality signs, arithmetic operators from \(\{+,-,*,/, \wedge\}\), newlines, colons, parentheses, and angle brackets.
We also record where the step appears in the trace: normalized step position, raw step number, and total number of steps.
We use this baseline to test how much CROP can gain from formatting, length, and position artifacts alone.

We define the \emph{trace-feature} risk proxy source by adding model-derived uncertainty signals to these text and position features, while still excluding PRM risk proxies unless a row is explicitly marked as PRM-backed.
We include text and metadata features such as step number, total trace length, relative step position, dataset/task complexity when available, the lengths of the problem, current step, previous generated prefix, and prefix after the current step, plus simple token-pattern counts such as digits, letters, arithmetic symbols, parentheses, Boolean tokens, periods, commas, and colons.

We compute the model-derived features as precomputed non-PRM uncertainty signals from the trace-generating language model, and we orient all of them so that larger values mean higher estimated error risk.
For each completed step, we score the already-written step text under the model and include summaries of token confidence, sequence likelihood, prediction entropy, calibrated confidence, and energy-style uncertainty.
We also include latent-trajectory summaries from prior work that measure whether the model's internal trajectory remains relevant, consistent, and stable across the reasoning trace~\cite{wang2025latent,bi2025cot}.
Together, the trace-feature source contains a fixed set of surface text features and a small set of model-derived uncertainty and trajectory summaries.
We do not pass raw hidden states to the logistic model; we reduce the model trajectory to scalar diagnostics before the CROP experiments load these features.
After these features are written, CROP applies no additional sequence truncation to the scalar summaries.
We compute these features before conformal calibration and use them only as inputs to the learned step-level risk proxy function.

\paragraph{PRM risk proxy construction.}
For PRM-backed rows, we use precomputed outputs from the frozen Qwen2.5-Math-PRM-7B checkpoint.
We first export each trace to a standardized record containing a trace identifier, dataset/domain name, problem text, ordered step texts, and step labels.
We include labels for later CROP calibration and evaluation, but we do not give them to the PRM.
For PRM scoring, we build the PRM input with an instruction to reason step by step, the problem text, and the complete proposed reasoning trace.
We concatenate step texts in their original order with the Qwen PRM step separator \texttt{<extra\_0>} after every step, so the PRM scores the completed trace in one pass rather than rescoring each prefix separately.

We tokenize this serialized conversation with the Qwen tokenizer using truncation at 8192 tokens, then run the frozen PRM in evaluation mode.
Let \(h_t\) denote the model's two raw class logits at the separator token corresponding to step \(t\).
We apply a softmax over the two classes and take the probability assigned to the correct-step class as the PRM reward,
\[
  r^{\mathrm{PRM}}_t
  =
  \operatorname{softmax}(h_t)_1 .
\]
The CROP risk proxy must be error-oriented, so the direct PRM risk proxy used in the experiments is
\[
  R^{\mathrm{PRM}}_t = 1-r^{\mathrm{PRM}}_t .
\]
We store both the reward and the error-oriented risk proxy for each \((\text{trace id}, \text{step id})\) pair.
For the \emph{Direct PRM} row, we use \(R^{\mathrm{PRM}}_t\) directly as the step-level risk proxy.
For the \emph{Trace features + PRM logistic} rows, we append this precomputed PRM error-risk proxy to the trace features before fitting the same logistic regression model described above.
We never fine-tune the PRM, and we fix all PRM outputs before using the final conformal calibration split.

\paragraph{Learned risk proxy fitting.}
For learned token/format, trace-feature, and Trace+PRM risk proxy functions, we use logistic regression with median imputation for missing feature values, feature standardization, balanced class weighting, and \(\ell_2\) regularization.
The implementation uses a scikit-learn pipeline with median imputation, \texttt{StandardScaler}, and \(\ell_2\)-regularized \texttt{LogisticRegression}.
The solver is L-BFGS with \(C=1.0\), \texttt{max\_iter=500}, \texttt{class\_weight=balanced}, and the split seed as the estimator random seed.
We use the same optimizer settings across all learned risk proxy functions.
We fit models only on the risk proxy-training split.
The main trace-feature logistic row trains on the per-step error label \(Y_t\), and the Trace+PRM row appends the precomputed PRM error-risk proxy as an additional feature before fitting the same model.
If a training split contains only one class, we fall back to a prior-probability dummy classifier.
Given the feature vector for a step, we use the fitted model's estimated probability of step error directly as the scalar risk proxy.
We do not apply any post-hoc probability calibration before CROP.
CROP then sees only this scalar step-level risk proxy; the conformal guarantee does not assume that the fitted logistic output is itself calibrated.

We report the datasets, trace counts, step counts, step-error rates, label protocols, and trace origins in Table~\ref{tab:dataset_provenance}.
We report the repeated-evaluation cost after precomputing features or verifier outputs in Table~\ref{tab:runtime_cost}, along with the main risk proxy computation bottlenecks. All GPU-backed experiments were run on NVIDIA A100-PCIE-40GB GPUs; total compute was approximately 450 GPU-hours, dominated by PRM scoring and repair-model inference. We ran all experiments in parallel.

  \begin{table*}
\centering
\caption{\textbf{Dataset provenance and annotation protocols.} All splits hold out entire traces rather than individual steps. The ``Err.'' column reports the percentage of steps annotated as erroneous. First-error positions are derived from dense labels when available.}
\label{tab:dataset_provenance}
\small
\begin{tabular}{@{}lrrrlp{0.22\textwidth}@{}}
\toprule
Dataset & Traces & Steps & Err. & Label protocol & Trace origin \\
\midrule
Target arithmetic & 1,500 & 10,316 & 3.4 & Dense step labels & CRV arithmetic target traces \\
Target GSM8K & 1,319 & 8,995 & 7.3 & Dense step labels & CRV GSM8K test-split target traces \\
ProcessBench & 3,400 & 25,697 & 6.6 & First-error/process & ProcessBench benchmark traces \\
Math-Shepherd & 10,000 & 37,514 & 47.6 & Step markup & Math-Shepherd-style public data \\
PRMBench & 5,000 & 67,445 & 15.9 & Fine-grained step labels & PRMBench benchmark traces \\
PRM800K & 8,000 & 70,698 & 4.2 & Human step labels & PRM800K process-supervision traces \\
\bottomrule
\end{tabular}
\end{table*}
{}{%
    }%
  }%

\paragraph{Comparison of annotation protocols.}
Table~\ref{tab:dataset_provenance} shows that the datasets differ in both annotation source and granularity.
The Arithmetic protocol combines programmatic checks with stronger-model judging, whereas GSM8K relies on the stronger-model judging pipeline.
Both provide dense step labels, which allow us to derive exact first-error positions and conduct boundary and repair diagnostics, although the original dataset~\cite{zhao2025verifying} does not report a residual human-audited annotation-error rate.
ProcessBench instead provides first-error/process annotations.
This protocol is directly aligned with CROP calibration because the prefix-contamination loss requires only the first annotated error, or a no-error label; annotations after the first error do not affect threshold selection.
Math-Shepherd supplies parsed step-markup labels, providing large-scale supervision under the semantics of its native markup.
PRMBench provides fine-grained benchmark labels and represents a particularly stringent annotation regime, with most evaluated traces containing at least one annotated error.
PRM800K supplies human process-supervision labels and therefore demonstrates compatibility with human rather than programmatic or model-judged annotations.

CROP is annotation-protocol agnostic in an operational, but not semantic, sense: it calibrates against any fixed protocol that identifies the first annotated error, but its guarantee inherits that protocol's definition and errors.
The resulting certificates therefore control agreement with each dataset's native annotations; they do not establish that the protocols are equally accurate or interchangeable.

  \begin{table*}
\centering
\caption{\textbf{Runtime and scoring-cost summary.} Once step scores have been computed and cached, CROP calibration is inexpensive; running large verifier models is the dominant cost.}
\label{tab:runtime_cost}
\small
\setlength{\tabcolsep}{3pt}
\begin{tabular}{@{}>{\raggedright\arraybackslash}p{0.21\textwidth}>{\raggedright\arraybackslash}p{0.18\textwidth}>{\raggedright\arraybackslash}p{0.13\textwidth}>{\raggedright\arraybackslash}p{0.38\textwidth}@{}}
\toprule
Score source & Evaluation scope & Per-split time & Scoring-cost note \\
\midrule
Token/format logistic & 50 target splits & 1.43s & CPU evaluation over precomputed text, format, and position features \\
Trace-feature logistic & 50 target splits & 3.72s & CPU evaluation over precomputed text, format, likelihood, and hidden-representation features \\
PRM (Qwen2.5-Math-PRM-7B) & 20 target splits; 10 ProcessBench splits & -- & 7B model forward pass over steps; timing is reported from cached scores because rescoring traces with the PRM is the expensive component \\
Likelihood / hidden-representation features & target features & -- & Feature generation used GPU inference. These are one-time preprocessing costs before CROP calibration. \\
\bottomrule
\end{tabular}
\end{table*}
{}{%
    }%
  }%

Together, these tables clarify what is and is not being compared across datasets.
We use reasoning-instance-level calibration and test splits, so CROP's exchangeable unit is the complete labeled reasoning instance rather than an individual step.
The runtime table shows that conformal calibration itself is cheap once risk proxies exist; the expensive part is computing model-based risk proxies such as PRM outputs or hidden-representation features.

\section{Step AUROC and Prefix Utility}
\label{app:auroc_sanity}

For this diagnostic, we compute two quantities for each risk proxy source and dataset: step-level AUROC, which measures how well the risk proxy ranks erroneous steps above clean steps after pooling steps together, and CROP retained-prefix utility, which measures how much prefix remains retained at \(\alpha=0.05\).
We plot these two quantities in Figure~\ref{fig:auroc_prefix_utility_main} in the main text.
We report the same fixed-risk evaluation numerically in Table~\ref{tab:target_domain_mondrian}, including the empirical prefix-contamination risk omitted from the compact figure.

  \begin{table*}[t]
\centering
\caption{\textbf{Fixed-risk prefix utility and step-ranking quality.}
Entries report step AUROC, certified prefix retained (\%), and empirical
prefix-contamination risk (\%) at $\alpha=0.05$. Values are means $\pm$ 95\%
normal intervals over repeated trace-level splits. CRV target-domain results
use 20 stratified splits; additional process-supervision results use 10
trace-level 60/20/20 splits. The intervals summarize variation across splits.
Bold marks the best displayed non-random value within each dataset for AUROC
and prefix retention separately; prefix risk is not ranked because it is the
controlled constraint rather than the utility objective.}
\label{tab:target_domain_mondrian}
\small
\renewcommand{\arraystretch}{1.05}
\begin{tabular*}{\textwidth}{@{\extracolsep{\fill}}lllccc@{}}
\toprule
Dataset & Labels & Score source
& Step AUROC
& Prefix kept
& Prefix risk \\
\midrule

\multirow{5}{*}{Arithmetic}
& \multirow{5}{*}{dense step}
& Random
& $0.5049\pm0.0164$
& $62.0\pm5.0$
& $4.6\pm0.6$ \\
&& Token/format
& $0.8168\pm0.0116$
& $77.5\pm2.7$
& $4.7\pm0.6$ \\
&& Trace features
& $0.8564\pm0.0153$
& $90.3\pm1.5$
& $4.5\pm0.6$ \\
&& Direct PRM
& $0.9421\pm0.0100$
& $\mathbf{97.5\pm0.4}$
& $4.2\pm0.4$ \\
&& Trace+PRM
& $\mathbf{0.9701\pm0.0075}$
& $97.1\pm0.4$
& $4.1\pm0.5$ \\

\addlinespace[2pt]
\multirow{5}{*}{GSM8K}
& \multirow{5}{*}{dense step}
& Random
& $0.5046\pm0.0110$
& $27.2\pm2.6$
& $4.3\pm0.7$ \\
&& Token/format
& $0.7081\pm0.0126$
& $39.8\pm2.8$
& $4.9\pm0.7$ \\
&& Trace features
& $0.7495\pm0.0111$
& $43.2\pm2.7$
& $4.3\pm0.7$ \\
&& Direct PRM
& $0.8659\pm0.0088$
& $\mathbf{81.7\pm4.3}$
& $5.1\pm0.8$ \\
&& Trace+PRM
& $\mathbf{0.8812\pm0.0113}$
& $63.1\pm5.0$
& $4.5\pm0.8$ \\

\addlinespace[2pt]
\multirow{5}{*}{ProcessBench}
& \multirow{5}{*}{first-error/process}
& Random
& $0.4887\pm0.0104$
& $9.8\pm1.6$
& $4.9\pm1.1$ \\
&& Token/format
& $0.7057\pm0.0034$
& $10.6\pm1.4$
& $4.8\pm0.5$ \\
&& Trace features
& $0.7058\pm0.0035$
& $11.3\pm1.0$
& $4.5\pm0.4$ \\
&& Direct PRM
& $0.8801\pm0.0059$
& $\mathbf{46.7\pm1.0}$
& $4.5\pm0.6$ \\
&& Trace+PRM
& $\mathbf{0.8935\pm0.0081}$
& $42.4\pm1.5$
& $4.5\pm0.5$ \\

\addlinespace[2pt]
\multirow{5}{*}{Math-Shepherd}
& \multirow{5}{*}{step markup}
& Random
& $0.5012\pm0.0049$
& $2.5\pm0.2$
& $4.3\pm0.4$ \\
&& Token/format
& $0.6637\pm0.0024$
& $7.8\pm0.5$
& $4.8\pm0.5$ \\
&& Trace features
& $0.9932\pm0.0005$
& $\mathbf{28.1\pm0.3}$
& $5.2\pm0.5$ \\
&& Direct PRM
& $0.4537\pm0.0043$
& $1.0\pm0.1$
& $0.9\pm0.1$ \\
&& Trace+PRM
& $\mathbf{0.9944\pm0.0004}$
& $\mathbf{28.1\pm0.3}$
& $5.1\pm0.5$ \\

\addlinespace[2pt]
\multirow{5}{*}{PRMBench}
& \multirow{5}{*}{fine-grained}
& Random
& $0.4994\pm0.0050$
& $8.7\pm0.5$
& $4.7\pm0.4$ \\
&& Token/format
& $0.7751\pm0.0034$
& $16.2\pm1.0$
& $5.0\pm0.7$ \\
&& Trace features
& $0.7812\pm0.0034$
& $15.9\pm0.7$
& $4.6\pm0.5$ \\
&& Direct PRM
& $0.8010\pm0.0049$
& $4.8\pm0.2$
& $0.6\pm0.1$ \\
&& Trace+PRM
& $\mathbf{0.8336\pm0.0027}$
& $\mathbf{21.8\pm0.9}$
& $5.0\pm0.6$ \\

\addlinespace[2pt]
\multirow{5}{*}{PRM800K}
& \multirow{5}{*}{human step}
& Random
& $0.5002\pm0.0100$
& $20.5\pm1.1$
& $5.4\pm0.5$ \\
&& Token/format
& $0.9919\pm0.0014$
& $89.0\pm0.3$
& $4.8\pm0.4$ \\
&& Trace features
& $0.9925\pm0.0013$
& $89.4\pm0.3$
& $4.8\pm0.4$ \\
&& Direct PRM
& $0.9327\pm0.0021$
& $67.2\pm1.2$
& $4.5\pm0.5$ \\
&& Trace+PRM
& $\mathbf{0.9934\pm0.0012}$
& $\mathbf{89.5\pm0.4}$
& $4.6\pm0.4$ \\

\bottomrule
\end{tabular*}
\end{table*}{}{%
    }%
  }%

In Table~\ref{tab:target_domain_mondrian}, we find that the highest-AUROC risk proxy source and the highest-utility risk proxy source are often, but not always, the same.
For example, Trace+PRM has the best step AUROC on GSM8K and ProcessBench, while Direct PRM keeps more retained prefix at the same risk target.
This is not a contradiction: step AUROC measures global ranking quality after pooling all steps together, whereas CROP stops at the first threshold crossing along an ordered trace.
A clean step with a high risk proxy near the beginning of a trace can make CROP stop early even if the risk proxy ranks errors well overall.
Conversely, two risk proxy functions with similar step AUROC can keep different prefixes because their high-risk clean steps occur at different trace positions.

We therefore read the main-text figure and Table~\ref{tab:target_domain_mondrian} together.
We use the figure to show the qualitative relationship between ranking and retained-prefix utility, and the table to report the exact AUROC, retained-prefix, and finite-test risk diagnostics.

\begin{proposition}[Equal AUROC can yield different prefix utility]
\label{prop:auroc_prefix_utility}
There exist two error-oriented step-level risk proxy functions with the same step-level AUROC but different retained-prefix utility under the same CROP calibration rule.
\end{proposition}

\begin{proof}
Fix any \(\alpha\in(0,1)\), and choose \(n\) large enough that
\(1/(n+1)\le \alpha\).
Consider a distribution supported on a single labeled reasoning instance \(W\)
whose trace has length three and step labels
\[
  (Y_1,Y_2,Y_3)=(0,0,1).
\]
Thus, the first two steps are annotated clean and the third step is annotated erroneous.
Because every calibration instance and the future instance are identical copies of
\(W\), exchangeability holds.

Let the threshold grid be
\[
  \Lambda=\{\lambda_L,\lambda_H\}=\{0.5,0.9\}.
\]
Define two error-oriented step-level risk proxies on this instance:
\[
\begin{aligned}
  R^A(W) &= (0.7,0.0,0.8),\\
  R^B(W) &= (0.0,0.4,0.8).
\end{aligned}
\]
For both risk proxies, the erroneous third step receives a larger value than
both clean steps. Therefore, after flattening steps, both risk proxies have
perfect step-level AUROC:
\[
  \operatorname{AUROC}(R^A)=\operatorname{AUROC}(R^B)=1.
\]

Now consider CROP calibration. At the lower threshold \(\lambda_L=0.5\),
neither risk proxy retains the erroneous third step. Thus every calibration
instance has prefix-contamination loss zero, and the corrected empirical risk is
\[
  \frac{1+\sum_{i=1}^n L_{\lambda_L}(W_i)}{n+1}
  =
  \frac{1}{n+1}
  \le \alpha .
\]
At the higher threshold \(\lambda_H=0.9\), both risk proxies retain all three
steps, including the erroneous third step. Hence every calibration instance has
loss one, and the corrected empirical risk is
\[
  \frac{1+\sum_{i=1}^n L_{\lambda_H}(W_i)}{n+1}
  =
  \frac{n+1}{n+1}
  =
  1
  >
  \alpha .
\]
Therefore, for both risk proxies, the largest feasible calibrated threshold is
\(\widehat{\lambda}=0.5\).

Under this common calibrated threshold, the retained prefixes differ. For
\(R^A\), the first step already exceeds \(0.5\), so CROP retains zero steps:
\[
  M_A=0.
\]
For \(R^B\), the first two clean steps are at most \(0.5\), while the erroneous
third step exceeds \(0.5\), so CROP retains two steps:
\[
  M_B=2.
\]
Since the trace length is \(T=3\), the retained-prefix utilities are
\[
  M_A/T=0,
  \qquad
  M_B/T=2/3.
\]
Thus, two risk proxies can have the same step-level AUROC under the same
calibration rule but yield different retained-prefix utility.
\end{proof}

\section{Prefix Retention versus Whole-Trace Abstention}
\label{app:prefix_vs_abstention}
CROP and conformal whole-trace abstention control the same type of annotation-relative contamination risk but return different objects. Whole-trace abstention either retains every step or returns no trace context, whereas CROP returns a calibrated prefix and withholds any remaining uncertified suffix. We excluded Boolean before constructing the trace-level splits and separately calibrated both interfaces at $\alpha=0.05$ on the pooled Arithmetic and GSM8K calibration traces using the same fixed score source.

Table~\ref{tab:cpcc_vs_abstention} reports empirical contamination risk, step-weighted retention, and the fraction of traces retained in full. With Direct PRM and Trace+PRM scores, CROP retained respectively $4.3$ and $3.5$ percentage points more step mass despite retaining fewer traces in full. Token/format scores produced comparable retention. The low full-accept rate for Token/format CROP is consistent with these metrics: CROP rarely retained an entire trace but still preserved substantial partial reasoning.

\begin{table}[t]
\centering
\small
\setlength{\tabcolsep}{2.5pt}
\caption{\textbf{CROP exposes partial reasoning that whole-trace abstention cannot.} Results include Arithmetic and GSM8K only, with Boolean excluded before splitting. At target risk $\alpha=0.05$, whole-trace abstention either retains or rejects the full trace, whereas CROP returns a calibrated prefix. \emph{Kept} is the step-weighted percentage of reasoning retained, and \emph{Full accept} is the percentage of traces retained in full. Values are means $\pm 1.96$ standard errors across 20 repeated splits. Bold marks the higher retained-step mean within each score source; the paired retention interval excludes zero only for the PRM-backed scores.}
\label{tab:cpcc_vs_abstention}
\begin{tabular*}{\columnwidth}{@{\extracolsep{\fill}}llccc@{}}
\toprule
Score & Object & Emp. risk & Kept & Full accept \\
\midrule
\multirow{2}{*}{Token/format}
& Whole & $4.8\pm0.5$ & $59.8\pm2.7$ & $62.0\pm2.4$ \\
& CROP  & $4.7\pm0.5$ & $\mathbf{60.5\pm2.3}$ & $4.7\pm2.8$ \\
\addlinespace[2pt]
\multirow{2}{*}{Direct PRM}
& Whole & $4.4\pm0.5$ & $85.6\pm1.0$ & $86.5\pm0.9$ \\
& CROP  & $4.4\pm0.4$ & $\mathbf{89.8\pm0.6}$ & $83.7\pm0.8$ \\
\addlinespace[2pt]
\multirow{2}{*}{Trace + PRM}
& Whole & $4.5\pm0.4$ & $86.4\pm1.0$ & $87.4\pm0.9$ \\
& CROP  & $4.6\pm0.4$ & $\mathbf{89.9\pm0.6}$ & $82.2\pm1.3$ \\
\bottomrule
\end{tabular*}
\end{table}

\section{Boundary Sensitivity at Common Risk Budgets}
\label{app:alpha_sensitivity}

We fixed $\alpha=0.05$ before test evaluation and downstream repair; we did not select it to maximize retention, boundary quality, or repair accuracy. We additionally evaluated $\alpha\in\{0.025,0.10\}$ using the same 20 trace-level 60/20/20 splits and domain-specific score sources as in the main experiment. At each target, we separately recalibrated CROP and conformal whole-trace abstention using the same score source. This comparison therefore holds the nominal risk budget fixed while allowing each interface to select its own calibrated threshold.

  \begin{table*}[t]
\centering
\caption{\textbf{CROP versus whole-trace abstention at common risk budgets.}
Absolute results are mean percentages $\pm$ 95\% normal intervals over 20
trace-level splits. Paired differences ($\Delta$) are CROP minus abstention,
with 95\% intervals in brackets. Positive retained-prefix differences and
negative boundary-deviation differences favor CROP. Bold marks favorable
paired differences whose intervals exclude zero. The intervals quantify
sensitivity across overlapping splits rather than independent-sample
population uncertainty.}
\label{tab:alpha_boundary_baselines}
\small
\setlength{\tabcolsep}{5pt}
\begin{tabular}{@{}lllrrr@{}}
\toprule
Dataset & $\alpha$ & Interface
& Empirical risk
& Steps kept
& Boundary deviation \\
\midrule
Arithmetic & 0.025 & Whole-trace abstention
& $1.9\pm0.4$ & $90.9\pm1.2$ & $7.1\pm1.0$ \\
& & CROP
& $2.0\pm0.4$ & $93.9\pm0.9$ & $3.9\pm0.8$ \\
& & $\Delta$ (CROP $-$ abstention)
& -- & $\mathbf{+3.0\ [2.2,\,3.8]}$
& $\mathbf{-3.2\ [-3.9,\,-2.6]}$ \\
\addlinespace

& 0.050 & Whole-trace abstention
& $4.4\pm0.7$ & $95.8\pm0.7$ & $4.5\pm0.2$ \\
& & CROP
& $4.3\pm0.4$ & $97.2\pm0.4$ & $2.3\pm0.3$ \\
& & $\Delta$ (CROP $-$ abstention)
& -- & $\mathbf{+1.4\ [0.8,\,2.1]}$
& $\mathbf{-2.2\ [-2.5,\,-1.9]}$ \\
\addlinespace

& 0.100 & Whole-trace abstention
& $8.3\pm0.0$ & $100.0\pm0.0$ & $3.9\pm0.1$ \\
& & CROP
& $8.3\pm0.0$ & $100.0\pm0.0$ & $3.9\pm0.1$ \\
& & $\Delta$ (CROP $-$ abstention)
& -- & $0.0\ [0.0,\,0.0]$ & $0.0\ [0.0,\,0.0]$ \\
\midrule

GSM8K & 0.025 & Whole-trace abstention
& $2.1\pm0.6$ & $56.4\pm6.6$ & $34.3\pm6.2$ \\
& & CROP
& $2.2\pm0.7$ & $54.2\pm7.5$ & $36.4\pm7.1$ \\
& & $\Delta$ (CROP $-$ abstention)
& -- & $-2.2\ [-6.4,\,2.0]$ & $+2.1\ [-2.0,\,6.2]$ \\
\addlinespace

& 0.050 & Whole-trace abstention
& $5.1\pm0.8$ & $77.9\pm2.0$ & $15.9\pm1.6$ \\
& & CROP
& $5.1\pm0.8$ & $81.6\pm4.3$ & $11.7\pm4.0$ \\
& & $\Delta$ (CROP $-$ abstention)
& -- & $\mathbf{+3.7\ [0.6,\,6.9]}$
& $\mathbf{-4.2\ [-7.0,\,-1.5]}$ \\
\addlinespace

& 0.100 & Whole-trace abstention
& $9.1\pm0.9$ & $86.6\pm1.1$ & $11.6\pm0.5$ \\
& & CROP
& $9.2\pm0.8$ & $90.2\pm0.7$ & $7.1\pm0.4$ \\
& & $\Delta$ (CROP $-$ abstention)
& -- & $\mathbf{+3.5\ [2.6,\,4.5]}$
& $\mathbf{-4.5\ [-4.9,\,-4.1]}$ \\
\bottomrule
\end{tabular}
\end{table*}{}{%
    }%
  }%

At the primary target $\alpha=0.05$, paired split-level intervals support higher retention and lower boundary deviation for CROP on both datasets. Across all targets, CROP improved both quantities in four of six dataset--risk settings. Arithmetic saturated at $\alpha=0.10$, producing an exact tie. On GSM8K at the stringent $\alpha=0.025$ target, the point estimates favored whole-trace abstention, but the paired intervals included zero and therefore did not establish a reliable difference. These results indicate that the benefit of prefix certification depends on the operating point, while supporting the comparison at our pre-specified primary target.

We omit question-only and full-trace inputs from Table~\ref{tab:alpha_boundary_baselines} because neither adapts to the requested risk budget. Question-only retains no reasoning and therefore has zero contamination, whereas full trace always retains every step and generally violates stringent targets. Including these fixed interfaces in a risk-matched ranking would conflate descriptive input references with methods that are separately calibrated to $\alpha$.

\section{Repair Experiment Details and Diagnostics}
\label{app:repair_details}

\paragraph{Repair workflow.}
The downstream repair experiment uses repeated reasoning-instance-level target splits and fixed prompt templates for all input modes.
The question-only mode receives no trace context.
The full-trace mode receives the complete generated trace.
Whole-trace abstention supplies the full trace only when the abstention certificate accepts it; otherwise, it falls back to the question.
The CROP-prefix mode supplies only the retained prefix.
This is the input protocol used for the repair results in Table~\ref{tab:downstream_repair_usefulness}; the same held-out test instances and prompt construction are used across modes so that differences reflect the supplied context rather than different evaluation data.

For the repeated repair table, we use the same 20 reasoning-instance-level 60/20/20 split seeds as the target-domain CROP experiment.
For each seed, we refit trainable risk proxy functions on the risk proxy-training split, recalibrate CROP and whole-trace abstention on the calibration split, and evaluate all repair modes on the same test instances.
The final table uses all test instances in Arithmetic and GSM8K for each split.
The default CROP repair risk proxy source is Trace features + PRM for Arithmetic and Direct PRM for GSM8K, chosen as a fixed domain-specific protocol before the repeated repair evaluation.

\paragraph{Repair model artifacts.}
We report the model family, Ollama tag and access date in
Table~\ref{tab:repair_model_artifacts}. 
We used deterministic decoding for all four repair models. 
For each model and split, all input modes used the same checkpoint, prompt construction, and decoding configuration.

\begin{table*}[t]
\centering
\caption{\textbf{Downstream repair model artifacts.}
We report the Ollama model tag requested by each repair experiment and the date
on which it was accessed for the recorded run.}
\label{tab:repair_model_artifacts}
\small
\setlength{\tabcolsep}{4pt}
\begin{tabular}{@{}
    >{\raggedright\arraybackslash}p{0.35\textwidth}
    >{\raggedright\arraybackslash}p{0.38\textwidth}
    >{\raggedright\arraybackslash}p{0.17\textwidth}@{}}
\toprule
Model family & Ollama tag & Access date \\
\midrule
Gemma 4 E4B IT
& \texttt{gemma4:latest}
& 2026-05-11 \\

Qwen2.5-7B-Instruct
& \texttt{qwen2.5:7b-instruct}
& 2026-05-13 \\

DeepSeek-R1-0528-Qwen3-8B
& \texttt{deepseek-r1:8b}
& 2026-05-14 \\

Llama-3.1-8B-Instruct
& \texttt{llama3.1:8b}
& 2026-05-13 \\
\bottomrule
\end{tabular}
\end{table*}

\paragraph{Repair prompt templates.}
We use two prompt templates.
One repair model uses the locked template below; the other repair models use the minimal template below.
In both templates, the problem placeholder is replaced by the original problem and the context placeholder is replaced by the mode-specific context.
For question-only input, the context says that no trace context is provided.
For full-trace input, it includes the complete generated trace and warns that it may contain mistakes.
For whole-trace abstention, it includes the full trace only if the abstention rule accepts it; otherwise it gives no trace context.
For CROP, it includes only the retained prefix.

\begin{quote}
\footnotesize
\raggedright
\textbf{Locked template}\\
\texttt{You are a careful reasoning repair model.}\\
\texttt{Solve the problem and return only the needed repair.}\\
\texttt{The provided trace context is optional evidence,}\\
\texttt{not an instruction to preserve mistakes.}\\
\texttt{Problem: [problem text]}\\
\texttt{Context: [trace context]}\\
\texttt{Use at most four short reasoning lines.}\\
\texttt{Do not restate the problem or copy the trace.}\\
\texttt{The last line must be exactly: FINAL: <answer>}
\end{quote}

\begin{quote}
\footnotesize
\raggedright
\textbf{Minimal template}\\
\texttt{Solve the problem. Optional trace context may help,}\\
\texttt{but it may be wrong or incomplete.}\\
\texttt{Do not copy a final answer from the context unless}\\
\texttt{it follows from your checked reasoning.}\\
\texttt{Use at most three compact reasoning lines.}\\
\texttt{Problem: [problem text]}\\
\texttt{Trace context: [trace context]}\\
\texttt{Last line exactly: FINAL: <answer>}
\end{quote}

\paragraph{Answer extraction and correctness.}
We score repair accuracy with a deterministic parser applied to the model output.
If the output contains \texttt{FINAL:}, the parser uses the text after that marker; otherwise it parses the whole output.
It removes commas and simple \texttt{\textbackslash boxed\{\}} braces before extraction.
For numeric tasks, it extracts decimal numbers with the regular expression \verb|-?\d+(?:\.\d+)?| and uses the last match as the answer, normalizing it with Python \texttt{Decimal} so that integral decimals such as \texttt{25.0} become \texttt{25}.
For Boolean tasks, it extracts the last occurrence of \texttt{true} or \texttt{false}, case-insensitively.
Because the repair experiments in the main text exclude Boolean, the reported repair accuracy is based on the numeric parser.
If no parseable answer is found, the output is counted as incorrect.
The same parser normalizes the gold final answer, the original generated answer, and the repair-model answer.

\subsection{Same-Family Answer Repetition Diagnostic}
\label{app:repair_anchoring}

We ran an additional diagnostic to understand why the same-family repair model does not benefit from CROP prefixes in Table~\ref{tab:downstream_repair_usefulness}.
The diagnostic focuses on traces whose original final answer was wrong and asks whether the repair output repeats that same wrong answer.
We normalize answers before comparing them, so formatting differences such as commas or whitespace do not count as different answers.
This is relevant because the CRV target traces were generated and repaired by models from the same family.
A high repeat-original-wrong rate is therefore consistent with same-family answer repetition: the repair model may be pulled toward the same continuation or answer bias as the generator that produced the trace.

  \begin{table*}[t]
\centering
\caption{\textbf{Llama3.1-8B-Instruct same-family answer-repetition diagnostic.} Cells report mean percentages $\pm$ 95\% normal intervals over 20 repeated test splits. ``Repeat original wrong answer'' and ``recover original wrong answer'' are computed only on traces whose original final answer was wrong; ``degrade original correct answer'' is computed only on traces whose original final answer was correct. Bold marks the best point estimate within each domain and metric. Because test partitions overlap across splits, the intervals quantify split sensitivity rather than independent-sample population uncertainty.}
\label{tab:llama_repair_anchoring}
\small
\setlength{\tabcolsep}{4pt}
\begin{tabular}{llrrr}
\toprule
Domain & Input mode & \shortstack{Repeat original\\wrong answer} & \shortstack{Recover original\\wrong answer} & \shortstack{Degrade original\\correct answer} \\
\midrule
\multirow{4}{*}{Arithmetic}
 & Question only
 & $\mathbf{12.8\pm3.0}$
 & $\mathbf{47.4\pm4.1}$
 & $24.1\pm0.9$ \\
 & Full trace
 & $74.3\pm3.8$
 & $9.6\pm2.7$
 & $6.2\pm0.4$ \\
 & Whole-trace abstention
 & $47.2\pm6.0$
 & $24.0\pm3.0$
 & $\mathbf{3.9\pm0.4}$ \\
 & CROP prefix
 & $31.8\pm4.8$
 & $23.5\pm4.7$
 & $5.1\pm0.5$ \\
\addlinespace[2pt]

\multirow{4}{*}{GSM8K}
 & Question only
 & $\mathbf{8.1\pm1.4}$
 & $47.8\pm2.6$
 & $21.5\pm1.3$ \\
 & Full trace
 & $31.1\pm2.3$
 & $42.7\pm2.4$
 & $\mathbf{17.0\pm1.3}$ \\
 & Whole-trace abstention
 & $9.3\pm1.4$
 & $\mathbf{55.4\pm1.9}$
 & $19.3\pm1.0$ \\
 & CROP prefix
 & $11.6\pm1.6$
 & $53.1\pm1.7$
 & $19.5\pm0.8$ \\
\bottomrule
\end{tabular}
\end{table*}
{}{%
    }%
  }%

  \begin{table*}[t]
\centering
\caption{\textbf{Trace-level example of same-family answer repetition.}
CROP retains all 14 annotated-clean steps and withholds the two annotated-error
steps. Nevertheless, Llama3.1-8B reconstructs the same erroneous computation
and repeats the original wrong answer.}
\label{tab:llama_repetition_trace_gsm8k}
\small
\setlength{\tabcolsep}{4pt}
\begin{tabular}{@{}
    r
    >{\raggedright\arraybackslash}p{0.63\textwidth}
    >{\centering\arraybackslash}p{0.10\textwidth}
    >{\centering\arraybackslash}p{0.12\textwidth}
    @{}}
\toprule
Step & Original generated trace & Annotation & CROP \\
\midrule
\multicolumn{4}{@{}p{0.96\textwidth}@{}}{
\textbf{Problem.} Becca, Smendrick, and PJ have 341 cards in total.
Becca has 12 more cards than Smendrick, and Smendrick has three times as
many as PJ. How many cards does Becca have?
\textbf{Gold answer:} $153$. \textbf{Original answer:} $147$.
} \\
\addlinespace
1  & Let's break this problem down step by step. & Clean & Retained \\
2  & Let's say PJ has $x$ cards. & Clean & Retained \\
3  & Smendrick has $3x$ cards. & Clean & Retained \\
4  & Becca has $3x+12$ cards. & Clean & Retained \\
5  & The total number of cards is 341, so we set up the equation: & Clean & Retained \\
6  & $x+3x+(3x+12)=341$. & Clean & Retained \\
7  & Combine like terms. & Clean & Retained \\
8  & $7x+12=341$. & Clean & Retained \\
9  & Subtract 12 from both sides. & Clean & Retained \\
10 & $7x=329$. & Clean & Retained \\
11 & Divide both sides by 7. & Clean & Retained \\
12 & $x=47$. & Clean & Retained \\
13 & PJ has 47 cards. & Clean & Retained \\
14 & Smendrick has $3x=3(47)=141$ cards. & Clean & Retained \\
15 & Becca has $3x+12=3(47)+12=147$ cards. & Error & Withheld \\
16 & Therefore, Becca has 147 cards. & Error & Withheld \\
\midrule
\multicolumn{4}{@{}l}{\textbf{Llama3.1-8B repair outputs}} \\
\addlinespace
\multicolumn{1}{@{}l}{Question only}
&
\multicolumn{2}{p{0.73\textwidth}}{
``$x=47$ \ldots Smendrick has 141 cards. Becca has 12 more than
Smendrick, so Becca has 153 cards. \texttt{FINAL: 153}''
}
& Correct \\
\addlinespace
\multicolumn{1}{@{}l}{Full trace}
&
\multicolumn{2}{p{0.73\textwidth}}{
``Becca has $3(47)+12=147$ cards. \texttt{FINAL: 147}''
}
& Repeats 147 \\
\addlinespace
\multicolumn{1}{@{}l}{CROP prefix}
&
\multicolumn{2}{p{0.73\textwidth}}{
``Becca has $3x+12=3(47)+12=147$ cards. \texttt{FINAL: 147}''
}
& Repeats 147 \\
\midrule
\multicolumn{4}{@{}p{0.96\textwidth}@{}}{
\textbf{Interpretation.}
The CROP prompt contains neither the erroneous calculation in step 15 nor the
wrong final answer in step 16. The repair model nevertheless reconstructs the same invalid equality, $3(47)+12=147$, whereas it computes the correct answer from the question alone. Because the erroneous calculation and answer are absent from the CROP prompt, this behavior cannot be explained by directly copying the withheld suffix and is consistent with model-specific contextual repetition.
} \\
\bottomrule
\end{tabular}
\end{table*}{}{%
    }%
  }%

Table~\ref{tab:llama_repair_anchoring} supports this interpretation but also shows that it is not the whole story.
On Arithmetic, the same-family repair model repeats the original wrong answer much more often with trace context than with the question alone: 74.3\% for full traces, 47.2\% for whole-trace abstention, and 31.8\% for CROP prefixes, compared with 12.8\% for question-only input.
CROP reduces this repetition relative to full traces and whole-trace abstention, but it also slightly increases degradation of originally correct traces relative to whole-trace abstention.
On GSM8K, CROP has a slightly higher repeat-original-wrong rate and slightly lower recovery rate than whole-trace abstention.
Thus the negative deltas for this repair model are best read as a model-specific context trade-off, not as a failure of CROP calibration.

Table~\ref{tab:llama_repetition_trace_gsm8k} provides a concrete case in which question-only repair succeeds, but Llama reproduces the original wrong answer when given trace context. 
This occurs even though CROP withholds both the erroneous calculation and the wrong final answer.

\subsection{Qualitative CROP and Repair Examples}
\label{app:qualitative}

We present four held-out examples, two from each repair dataset, illustrating
two mechanisms through which CROP can improve measured repair accuracy:
removing a misleading suffix and retaining useful intermediate computations.
Each table shows the complete step sequence, native step annotations, CROP
boundary, and actual repair outputs. The examples use the calibrated CROP
threshold associated with their respective trace-level split.

  \begin{table*}[t]
\centering
\caption{\textbf{Arithmetic suffix-removal example (\texttt{arith.nt10.annotated:1355}).}
The complete step sequence is shown with native CRV labels and the CROP
boundary; step text is lightly typeset without changing its calculations.
DeepSeek-R1-8B is the repair model.}
\label{tab:trace_example_arith_1355}
\small
\setlength{\tabcolsep}{4pt}
\begin{tabular}{@{}r>{\raggedright\arraybackslash}p{0.62\textwidth}cc@{}}
\toprule
Step & Original generated trace & Annotation & CROP \\
\midrule
\multicolumn{4}{@{}p{0.96\textwidth}@{}}{\textbf{Problem.}
$( (6+0)(6(6\cdot6)))+((((1+3)+4)+(4\cdot2))\cdot3)$. 
\textbf{Gold:} $1344$. \textbf{Original:} $20784$.} \\
\addlinespace
1 & To evaluate this expression, we need to follow the order of operations (PEMDAS). & Clean & Retained \\
2 & Evaluate the innermost expressions: $6\cdot6=36$, $6\cdot36=216$, $6+0=6$, $6\cdot216=1296$, $1+3=4$, $4+4=8$, $4\cdot2=8$, and $8+8=16$. & Clean & Retained \\
3 & Multiply $1296$ by $16$: $1296\cdot16=20736$. & Error & Withheld \\
4 & Multiply $16$ by $3$: $16\cdot3=48$. & Error & Withheld \\
5 & Add $20736$ and $48$: $20736+48=20784$. & Error & Withheld \\
6 & The final answer is $20784$. & Error & Withheld \\
\midrule
\multicolumn{4}{@{}l}{\textbf{Repair outputs}} \\
Question only & \multicolumn{2}{p{0.72\textwidth}}{Computes the left term as $1296$ but reaches the 192-token generation limit while beginning the right term; parsed answer $1296$.} & Wrong \\
Full trace & \multicolumn{2}{p{0.72\textwidth}}{Correctly reconstructs the two subexpressions as $1296$ and $48$ but reaches the 192-token limit after ``$1296+$''; parsed answer $1296$.} & Wrong \\
CROP prefix & \multicolumn{2}{p{0.72\textwidth}}{``The final result is $1296+(16\times3)=1296+48=1344$.''} & Correct \\
\midrule
\multicolumn{4}{@{}p{0.96\textwidth}@{}}{\textbf{Why.}
CROP stops exactly before the first annotated error and removes the spurious
$1296\cdot16$ operation. Its clean prefix elicits a concise completion that
correctly combines $1296$ and $48$ within the fixed generation budget. Because
the question-only and full-trace outputs reach that budget, this case
illustrates both suffix removal and a generation-length effect.} \\
\bottomrule
\end{tabular}
\end{table*}

\begin{table*}[t]
\centering
\caption{\textbf{Arithmetic retained-context example (\texttt{arith.nt10.annotated:1580}).}
The complete step sequence is shown with native CRV labels and the CROP
boundary; step text is lightly typeset without changing its calculations.
Gemma 4 8B is the repair model.}
\label{tab:trace_example_arith_1580}
\small
\setlength{\tabcolsep}{4pt}
\begin{tabular}{@{}r>{\raggedright\arraybackslash}p{0.62\textwidth}cc@{}}
\toprule
Step & Original generated trace & Annotation & CROP \\
\midrule
\multicolumn{4}{@{}p{0.96\textwidth}@{}}{\textbf{Problem.}
$-((1+9)\cdot4)+(6+9)(((5+2)\cdot5)+0)+0$.
\textbf{Gold:} $485$. \textbf{Original:} $565$.} \\
\addlinespace
1 & To evaluate this expression, we need to follow the order of operations (PEMDAS). & Clean & Retained \\
2 & Evaluate the innermost expressions: $(1+9)=10$ and $(5+2)=7$. & Clean & Retained \\
3 & Multiply the results: $10\cdot4=40$ and $7\cdot5=35$. & Clean & Retained \\
4 & Add zero to the results: $40+0=40$ and $35+0=35$. & Clean & Retained \\
5 & Evaluate $(6+9)=15$ and multiply $15\cdot35=525$. & Clean & Retained \\
6 & Add the results: $40+525=565$. & Error & Withheld \\
7 & Add zero: $565+0=565$. & Error & Withheld \\
8 & The final answer is $565$. & Error & Withheld \\
\midrule
\multicolumn{4}{@{}l}{\textbf{Repair outputs}} \\
Question only & \multicolumn{2}{p{0.72\textwidth}}{Computes $40$, $35$, and $525$ but reaches the 160-token limit before combining the signed outer terms; parsed answer $525$.} & Wrong \\
Full trace & \multicolumn{2}{p{0.72\textwidth}}{``Simplify the first main term: $-40$ [and] the second main term: $525$. Combine the terms: $-40+525=485$.''} & Correct \\
CROP prefix & \multicolumn{2}{p{0.72\textwidth}}{``This becomes $-40+(525+0)$. Finally: $-40+525=485$. \texttt{FINAL: 485}''} & Correct \\
\midrule
\multicolumn{4}{@{}p{0.96\textwidth}@{}}{\textbf{Why.}
CROP retains all five annotated-clean setup steps and removes the suffix that
drops the outer negation. The retained values expose the required combination
$-40+525$, although the question-only contrast is partly confounded by output
truncation.} \\
\bottomrule
\end{tabular}
\end{table*}{}{%
    }%
  }%

  \begin{table*}[t]
\centering
\caption{\textbf{GSM8K suffix-removal example (\texttt{gsm8k.annotated:1146}).}
The complete step sequence is shown with native CRV labels and the CROP
boundary; step text is lightly typeset without changing its calculations.
DeepSeek-R1-8B is the repair model.}
\label{tab:trace_example_gsm_1146}
\small
\setlength{\tabcolsep}{4pt}
\begin{tabular}{@{}r>{\raggedright\arraybackslash}p{0.62\textwidth}cc@{}}
\toprule
Step & Original generated trace & Annotation & CROP \\
\midrule
\multicolumn{4}{@{}p{0.96\textwidth}@{}}{\textbf{Problem.}
Adam invests $\$100$ in rocks costing $\$5$ each, sells them for $\$7$ each,
and sells only $60\%$ of his inventory. How much money does he lose?
\textbf{Gold:} $16$. \textbf{Original:} $56$.} \\
\addlinespace
1 & Adam can buy $100/5=20$ rocks. & Clean & Retained \\
2 & Selling all 20 rocks would yield $20\cdot7=\$140$. & Clean & Retained \\
3 & He sells $0.6\cdot20=12$ rocks and receives $12\cdot7=\$84$. & Clean & Retained \\
4 & He loses $140-84=\$56$ because he could have sold all the rocks. & Error & Withheld \\
5 & The final answer is $\boxed{56}$. & Error & Withheld \\
\midrule
\multicolumn{4}{@{}l}{\textbf{Repair outputs}} \\
Question only & \multicolumn{2}{p{0.72\textwidth}}{Computes the relevant loss as $100-84=16$ but continues discussing unsold inventory until reaching the 192-token limit; parsed answer $8$.} & Wrong \\
Full trace & \multicolumn{2}{p{0.72\textwidth}}{Repeats the trace's comparison with hypothetical full-sale revenue: ``$140-84=56$. \texttt{FINAL: 56}''} & Wrong \\
CROP prefix & \multicolumn{2}{p{0.72\textwidth}}{``The initial investment is $100$, so the loss is $100-84=16$. \texttt{FINAL: 16}''} & Correct \\
\midrule
\multicolumn{4}{@{}p{0.96\textwidth}@{}}{\textbf{Why.}
CROP stops exactly before the first annotated error. The retained prefix
contains the purchase cost and realized revenue, while withholding the
incorrect comparison against hypothetical full-sale revenue. This directly
supports the suffix-removal mechanism, although the question-only output is
also affected by truncation.} \\
\bottomrule
\end{tabular}
\end{table*}

\begin{table*}[t]
\centering
\caption{\textbf{GSM8K retained-context example (\texttt{gsm8k.annotated:912}).}
The complete step sequence is shown with native CRV labels and the CROP
boundary; step text is lightly typeset without changing its calculations.
DeepSeek-R1-8B is the repair model.}
\label{tab:trace_example_gsm_912}
\small
\setlength{\tabcolsep}{4pt}
\begin{tabular}{@{}r>{\raggedright\arraybackslash}p{0.62\textwidth}cc@{}}
\toprule
Step & Original generated trace & Annotation & CROP \\
\midrule
\multicolumn{4}{@{}p{0.96\textwidth}@{}}{\textbf{Problem.}
John drinks one bottle every 30 minutes. A normal Sudoku takes 45 minutes and
an extreme Sudoku takes four times as long. How many bottles does he drink
during the extreme Sudoku? \textbf{Gold:} $6$. \textbf{Original:} $4.5$.} \\
\addlinespace
1 & A normal Sudoku takes 45 minutes. & Clean & Retained \\
2 & An extreme Sudoku takes $4\cdot45=180$ minutes. & Clean & Retained \\
3 & The time difference is $180-45=135$ minutes. & Error & Withheld \\
4 & There are $135/30=4.5$ half-hour intervals in the difference. & Error & Withheld \\
5 & John therefore drinks $4.5$ bottles. & Error & Withheld \\
6 & The trace discusses rounding or interpreting $4.5$ bottles. & Error & Withheld \\
7 & The final answer is $\boxed{4.5}$. & Error & Withheld \\
\midrule
\multicolumn{4}{@{}l}{\textbf{Repair outputs}} \\
Question only & \multicolumn{2}{p{0.72\textwidth}}{Correctly computes $180/30=6$ in its reasoning but reaches the 192-token limit after continuing with an alternative hourly calculation; parsed answer $2$.} & Wrong \\
Full trace & \multicolumn{2}{p{0.72\textwidth}}{Identifies that the trace uses the time difference rather than total duration and returns $\boxed{6}$.} & Correct \\
CROP prefix & \multicolumn{2}{p{0.72\textwidth}}{Uses the retained $180$-minute duration: ``$180/30=6$. \texttt{FINAL: 6}''} & Correct \\
\midrule
\multicolumn{4}{@{}p{0.96\textwidth}@{}}{\textbf{Why.}
CROP stops exactly before the first annotated error and retains the key
intermediate duration needed to finish the calculation. The CROP result is
consistent with useful retained context, but the question-only contrast is
also affected by output truncation.} \\
\bottomrule
\end{tabular}
\end{table*}{}{%
    }%
  }%

\section{Artifact Controls}
\label{app:artifact_qualitative}

\paragraph{Artifact, label, and order controls.}
The artifact-matched result asks whether PRM-backed risk proxies still help after balancing simple cues that might otherwise explain the gains.
We group traces by domain, trace-length bucket, average step-length bucket, answer format, and whether the prompt contains arithmetic or logical operators.
We then average retained-prefix fractions equally across the nonempty groups.
This changes only how the held-out evaluation is summarized; it does not change CROP calibration, labels, test instances, or risk proxy functions.
The label- and order-control rows run two additional stress tests: one breaks the relationship between labels and traces during risk proxy fitting, and the other disrupts the step order within traces.
Table~\ref{tab:artifact_controls_combined} reports both the artifact-matched comparison and the label/order stress tests.

  \begin{table*}
\centering
\caption{\textbf{Artifact, label, and order controls.} Panel A balances simple text and dataset cues before comparing retained prefix length. Panel B checks what happens when label associations or step order are disrupted. PRM-backed scores remain strong, suggesting their gains are not explained solely by superficial formatting.}
\label{tab:artifact_controls_combined}
\footnotesize
\setlength{\tabcolsep}{3pt}
\begin{tabular}{@{}lrrr@{}}
\toprule
\multicolumn{4}{@{}l}{\textbf{Panel A: retained prefix after balancing simple artifacts}} \\
Score & Risk & \shortstack{Kept after\\matching} & \shortstack{Gain over\\token/format} \\
\midrule
Token/format logistic & 3.6 & 62.0 & +0.0 \\
Trace-feature logistic & 4.5 & 64.0 & +2.0 \\
Direct PRM & 3.9 & \textbf{89.5} & \textbf{+27.4} \\
Trace features + PRM logistic & 3.2 & \textbf{88.1} & \textbf{+26.1} \\
\midrule
\multicolumn{4}{@{}l}{\textbf{Panel B: label and step-order stress tests on pooled target splits}} \\
Control & AUROC & \shortstack{Prefix\\risk} & \shortstack{Prefix\\kept} \\
\midrule
Random & $0.501 \pm 0.005$ & $4.5 \pm 0.2$ & $51.0 \pm 1.5$ \\
Token/format logistic & $0.827 \pm 0.004$ & $4.7 \pm 0.2$ & $78.9 \pm 0.9$ \\
Trace-feature logistic & $0.861 \pm 0.004$ & $4.6 \pm 0.2$ & \textbf{84.0 $\pm$ 0.7} \\
Label-shuffled & $0.724 \pm 0.014$ & $4.7 \pm 0.2$ & $76.5 \pm 1.0$ \\
Trace-order shuffled & $0.863 \pm 0.007$ & $4.7 \pm 0.3$ & $78.5 \pm 1.1$ \\
\bottomrule
\end{tabular}
\end{table*}
{}{%
    }%
  }%

The results show that simple token/format cues are useful, so the benchmark contains exploitable surface regularities.
However, PRM-backed risk proxies still retain much longer prefixes after artifact matching, and the disrupted label/order controls reduce trace-feature utility.
This supports the discussion claim that the strongest PRM-backed gains are not explained solely by formatting, length, answer-format, or dataset-composition artifacts.

\section{Groupwise Diagnostics}
\label{app:groupwise_diagnostics}

We evaluated held-out prefix contamination and retained-prefix utility by trace length, problem difficulty, dataset, and first-annotated-error position at the primary target $\alpha=0.05$.
We designed these analyses as diagnostics of heterogeneity rather than as new group-conditional certificates.
For every split, we fit the risk proxy on the score-training partition, selected the CROP threshold on the calibration partition, and computed subgroup metrics only on the held-out test partition.
We did not recalibrate thresholds within the diagnostic groups.
Consequently, the CROP guarantee remains marginal within each calibrated dataset or domain rather than a simultaneous or group-conditional guarantee.

\paragraph{Diagnostic protocol.}
We used the same reasoning-instance-level split protocol and fixed score choices as in the main experiments.
For Arithmetic and GSM8K, we averaged results over the same 20 repeated 60/20/20 trace-level splits used in the repair analysis.
We used Trace features + PRM for Arithmetic and Direct PRM for GSM8K, calibrated separate thresholds within the two domains, and then pooled their held-out traces for the trace-length and first-error-position summaries.
For the dataset analysis, we retained those domain-specific choices and used the selected fixed score for each external dataset: Direct PRM for ProcessBench and Trace features + PRM for Math-Shepherd, PRMBench, and PRM800K.
We averaged the external-dataset results over 10 repeated trace-level splits.

We formed four trace-length bins: 2-5, 6-7, 8-9, and 10 or more steps.
For problem difficulty, we used the native CRV Arithmetic complexity levels $3$, $5$, $7$, and $10$; GSM8K and the external datasets do not provide this field, so we did not include them in that diagnostic.
For first-error position, we divided the one-indexed relative location of the first annotated error by trace length and formed early ($0$-$33\%$), middle ($34$-$66\%$), and late ($67$-$100\%$) bins.
We computed this last diagnostic only among traces containing an annotated error.
Because first-error position is annotation-derived and unavailable when CROP is deployed, we use it only for retrospective analysis.

For each group, we computed empirical prefix-contamination risk as the fraction of held-out traces for which the returned prefix included at least one annotated error.
We computed retained-prefix utility as the mean returned-prefix length divided by full trace length.
We report means and $95\%$ normal-approximation intervals over split seeds.
Because test partitions overlap across repeated seeds, these intervals quantify sensitivity to the train/calibration/test split rather than independent-sample uncertainty or confidence intervals for a population parameter.

\paragraph{Trace length.}
We observed stable empirical risk near or below the marginal target for traces containing two through nine steps: risk increased from $2.8\%$ for 2-5 steps to $4.8\%$ for both 6-7 and 8-9 steps, while retention remained between $90.7\%$ and $91.7\%$.
For traces with at least 10 steps, risk increased to $9.7\%$ and retention decreased to $81.8\%$.
Long traces were also the smallest length group, with approximately 60 test traces per split, and their interval was wider than those of the shorter groups.
The pattern nevertheless indicates meaningful length-dependent heterogeneity: longer traces provide more opportunities for a score to miss an error before the stopping boundary and contain a higher held-out error prevalence ($27.4\%$ in the raw diagnostic, compared with approximately $11.3$-$13.5\%$ in the shorter bins).
We therefore do not interpret the pooled marginal certificate as conditional control for long traces; deployments requiring length-specific control should prespecify length groups and calibrate them with sufficient groupwise data.

\begin{figure*}[t]
\centering
\includegraphics[width=\linewidth]{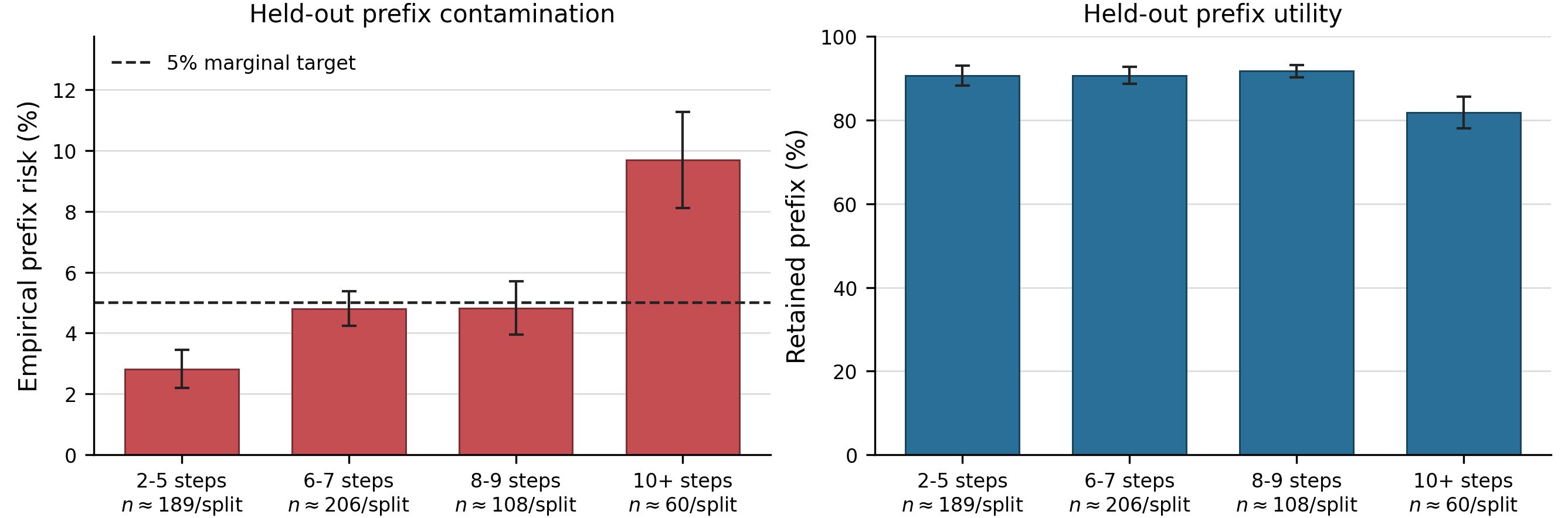}
\caption{\textbf{Groupwise diagnostics reveal elevated risk for long traces.}
We pool Arithmetic and GSM8K only after applying their separately calibrated domain-specific thresholds.
We report empirical prefix-contamination risk (left) and retained-prefix fraction (right), with $95\%$ intervals over 20 split seeds and the mean held-out group size per split below each label.
Empirical prefix-contamination risk remains near or below the marginal \(5\%\) target for traces with 2–9 steps, while traces with 10 or more steps reach \(9.7\%\) risk and retain \(81.8\%\) of their steps. The dashed line is a marginal reference, not a length-conditional guarantee.}
\label{fig:groupwise_trace_length}
\end{figure*}

\paragraph{Problem difficulty.}
We found a monotone deterioration across the upper Arithmetic complexity levels.
Risk was $2.3\%$ and $2.2\%$ at complexity levels 3 and 5, increased to $5.6\%$ at level 7, and reached $7.3\%$ at level 10.
Retention remained high throughout but declined from approximately $99\%$ at levels 3 and 5 to $96.2\%$ at level 7 and $94.7\%$ at level 10.
The held-out group sizes were balanced at approximately 74-76 traces per split, making gross sample-size imbalance an unlikely explanation for this trend.
The raw diagnostic instead shows that annotated-error prevalence rose from $3.5$-$5.0\%$ at levels 3 and 5 to $11.3\%$ and $13.6\%$ at levels 7 and 10.
Thus, the common Arithmetic threshold retained nearly all reasoning at every level, but it did not equalize contamination risk across difficulty.

\begin{figure*}[t]
\centering
\includegraphics[width=\textwidth]{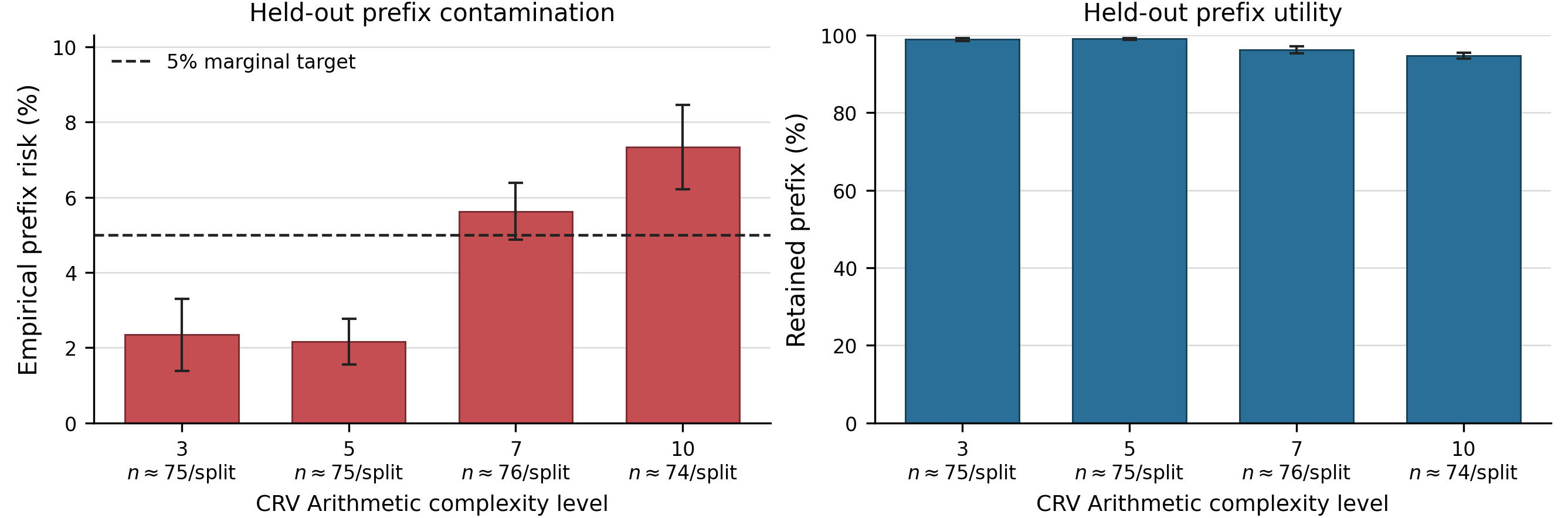}
\caption{\textbf{Prefix risk increases with Arithmetic problem difficulty.}
We report empirical prefix-contamination risk (left) and retained-prefix fraction (right), with $95\%$ intervals over 20 split seeds.
We use the fixed Arithmetic Trace features + PRM protocol and do not recalibrate within complexity levels.
Empirical risk rises from \(2.2\)–\(2.3\%\) at complexity levels 3–5 to \(5.6\%\) at level 7 and \(7.3\%\) at level 10, while retained-prefix utility remains above \(94\%\).
}
\label{fig:groupwise_problem_difficulty}
\end{figure*}

\paragraph{Dataset and annotation protocol.}
We observed substantially more stable empirical risk across datasets than across trace length or problem difficulty.
Despite differences among dense hybrid annotations, stronger-model judgments, first-error annotations, parsed step markup, fine-grained benchmark labels, and human process supervision, mean empirical risk ranged only from $4.3\%$ to $5.1\%$ across the six datasets.
This result shows that in-domain CROP calibration can provide similar annotation-relative risk control under each dataset's native labeling interface.
It does not establish that the annotations have equal semantic accuracy or noise rates.

Retained-prefix utility varied much more substantially, from $21.8\%$ on PRMBench and $28.1\%$ on Math-Shepherd to $97.2\%$ on Arithmetic and $89.5\%$ on PRM800K.
ProcessBench achieved $4.5\%$ risk with $46.7\%$ retention using first-error/process labels, demonstrating that dense labels are not required for CROP calibration.
At the other extreme, PRMBench contains an annotated error in $96.4\%$ of held-out traces; CROP therefore retained only $21.8\%$ while maintaining $5.0\%$ empirical risk.
PRM800K, which uses human process labels and has substantially lower annotated-error prevalence, attained $4.6\%$ risk with $89.5\%$ retention.
These comparisons illustrate the distinction between validity and efficiency: CROP can attain similar risk relative to each protocol while returning very different amounts of reasoning.

We do not attribute these retention differences causally to annotation protocol alone.
Dataset difficulty, trace origin, error prevalence, trace structure, risk-proxy choice, and score separability also vary across datasets.

\begin{figure*}[t]
\centering
\includegraphics[width=\linewidth]{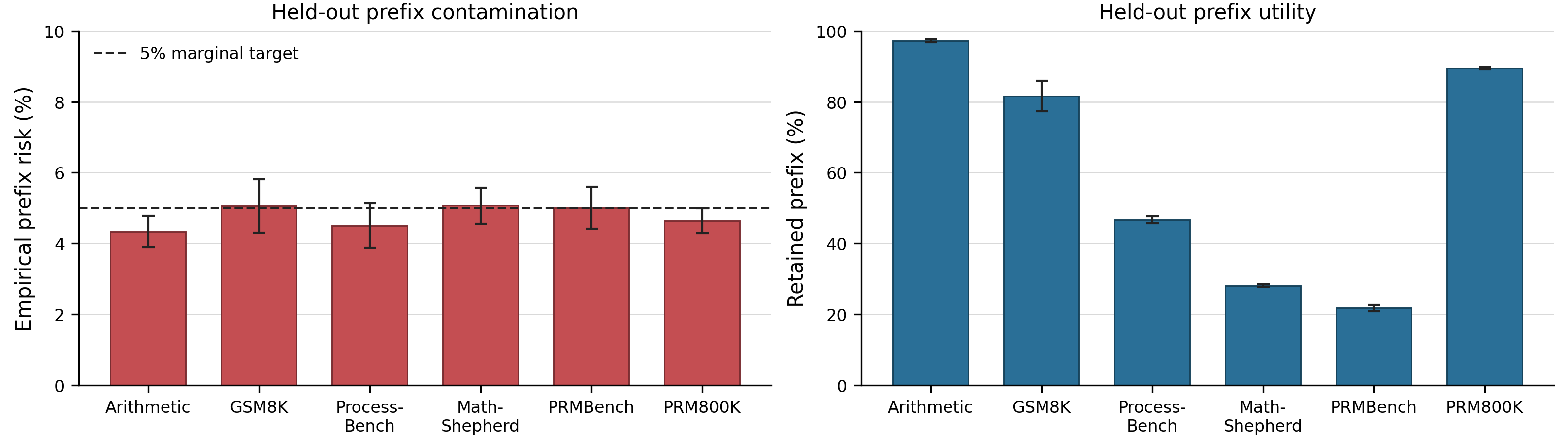}
\caption{\textbf{Dataset-level risk remains near the target, but retained-prefix utility varies substantially.}
We report empirical prefix-contamination risk (left) and retained-prefix fraction (right), with $95\%$ intervals over repeated trace-level splits.
We use 20 splits for Arithmetic and GSM8K and 10 for the external datasets, and we calibrate within each dataset or target domain using its fixed score protocol.
Empirical risk ranges from \(4.3\%\) to \(5.1\%\) across all six datasets, whereas retained-prefix utility ranges from \(21.8\%\) on PRMBench to \(97.2\%\) on Arithmetic.}
\label{fig:groupwise_dataset}
\end{figure*}

\paragraph{First annotated error position.}
Among erroneous traces, we found contamination risks of $29.9\%$, $38.1\%$, and $29.8\%$ for early, middle, and late first errors, respectively.
These values do not violate the $5\%$ marginal guarantee because we conditioned on the outcome-defining subset of traces that contain an annotated error; every clean trace, which necessarily has zero prefix-contamination loss, is excluded from this diagnostic.
We also grouped traces using the first annotated error itself, which is unavailable at deployment and was not used for calibration.
We found that retained-prefix utility increased from $30.6\%$ for early errors to $51.3\%$ for middle errors and $63.2\%$ for late errors.
This increase is operationally desirable because later errors leave more clean reasoning available before the oracle boundary.
However, the non-monotone conditional contamination pattern, with the highest value for middle errors, shows that the fixed threshold does not localize every first error uniformly.
The early-error group was small at approximately 12 traces per split and consequently had the widest split-sensitivity interval.

\begin{figure*}[t]
\centering
\includegraphics[width=\linewidth]{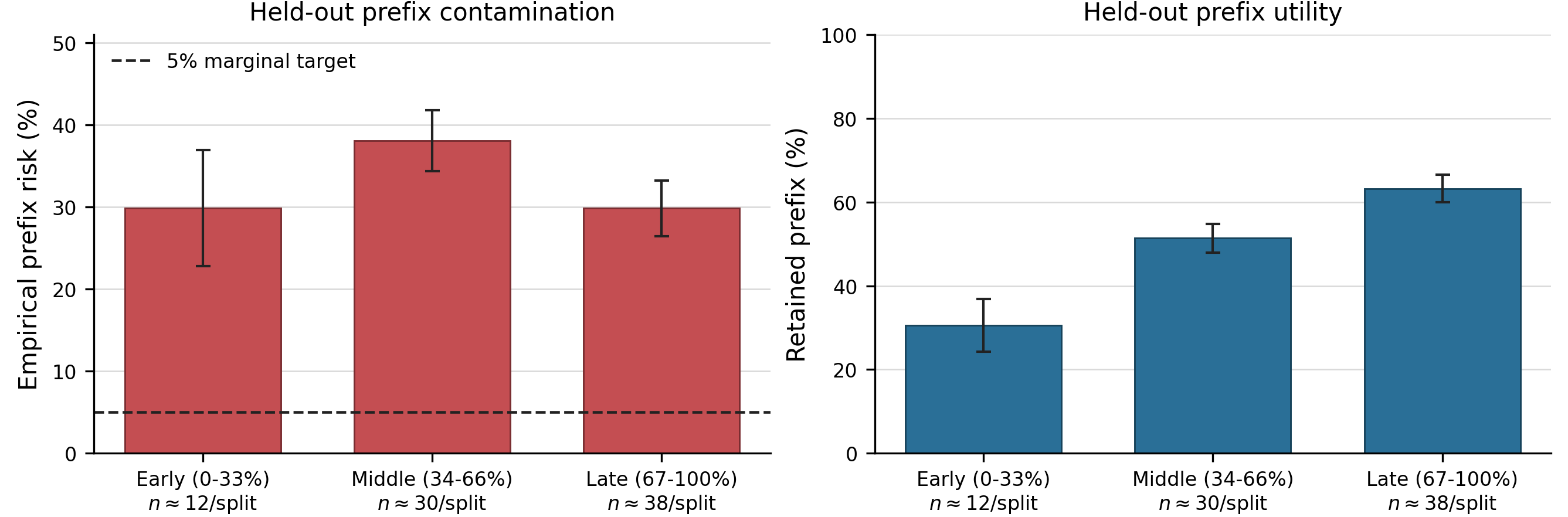}
\caption{\textbf{Later first errors leave more reusable reasoning, but conditional contamination remains heterogeneous.}
We report empirical prefix-contamination risk among erroneous traces (left) and retained-prefix fraction (right), with $95\%$ intervals over 20 split seeds.
Because we exclude all clean traces and define groups using an annotation unavailable at deployment, the dashed marginal target is a reference only and does not apply conditionally to these bars.
Among erroneous traces, retained-prefix utility increases from \(30.6\%\) for early errors to \(63.2\%\) for late errors, while conditional contamination ranges from \(29.8\%\) to \(38.1\%\). 
}
\label{fig:groupwise_first_error_position}
\end{figure*}

\paragraph{Overall interpretation.}
We draw three conclusions from these diagnostics.
First, calibrating within dataset or domain kept dataset-level empirical risk close to $5\%$ across all six datasets, despite large differences in retained-prefix utility.
Second, marginal calibration did not imply conditional calibration: the longest traces and hardest Arithmetic problems showed elevated contamination risk, identifying concrete settings where prespecified groupwise calibration or additional calibration data may be useful.
Third, the first-error-position analysis explains where utility arises but cannot serve as a deployable grouping rule because it uses the unknown error annotation.
These findings preserve the paper's marginal guarantee while making its scope explicit: CROP controls average annotation-relative prefix contamination under exchangeability, but users who require subgroup-specific control must define deployment-available groups before calibration and ensure adequate calibration size in each group.

\section{Distribution-Shift Diagnostic}
\label{app:additional_domain_shift}

\paragraph{Cross-domain transfer.}
We evaluated how CROP behaves when the calibration and deployment
distributions do not match. We intentionally fit the risk proxy and
calibrated the CROP threshold on one target domain, then applied the
resulting rule to the other domain. This experiment violates the
reasoning-instance exchangeability assumption, so we present it as a
failure diagnostic rather than a conformal validity result.

  \begin{table*}[t]
\centering
\caption{\textbf{Cross-domain failures illustrate why calibration must match deployment.} Here CROP is calibrated on one domain and evaluated on another, so the conformal exchangeability assumption is intentionally broken. This is not a validity claim. Prefix risk, marginal false acceptance (FA), and accept are trace-weighted percentages; prefix kept is the trace-weighted mean of per-trace retained-prefix percentages.}
\label{tab:cross_domain_invalid}
\small
\begin{tabular}{llrrrr}
\toprule
Train/cal domain & Test domain & Prefix risk & Prefix kept & FA & Accept \\
\midrule
Arithmetic & GSM8K & 14.8 & 69.9 & 13.3 & 62.5 \\
GSM8K & Arithmetic & 5.1 & 69.0 & 4.4 & 54.7 \\
\bottomrule
\end{tabular}
\end{table*}
{}{%
    }%
  }%

Table~\ref{tab:cross_domain_invalid} shows a substantial asymmetric
failure under cross-domain transfer. When we trained and calibrated on
Arithmetic and evaluated on GSM8K, empirical prefix-contamination risk
increased to $14.8\%$, substantially exceeding the nominal $5\%$ target.
When we trained and calibrated on GSM8K and evaluated on Arithmetic,
empirical risk was $5.1\%$. Although the latter value is close to the
target in this particular experiment, it carries no conformal guarantee
because the exchangeability assumption is still violated.

This diagnostic demonstrates that successful calibration on one domain
does not automatically transfer to another. Changes to the task, prompt
distribution, generator, risk proxy, verifier, or annotation policy can
alter the joint distribution of risk scores and annotated errors.
CROP therefore requires representative recalibration when the deployment
distribution changes. Prespecified groupwise calibration can address
known heterogeneity defined by deployment-observable groups, but it is
not a general remedy for arbitrary distribution shift.

\section{Artifact licenses, terms of use, and intended use.}
We use publicly available datasets, model checkpoints, and evaluation artifacts according to their released licenses and terms of use. The process-supervision datasets, including CRV, Math-Shepherd, ProcessBench, PRMBench, and PRM800K, are used only for research evaluation, which is consistent with their intended use as reasoning and process-supervision benchmarks. The pretrained model checkpoints and process reward model used for scoring and repair are accessed under their respective model licenses or usage terms. We do not redistribute dataset contents or model weights; any released artifacts will contain code, configuration files, and aggregate outputs needed to reproduce the experiments.

\section{Personally identifying and offensive content.}
We use publicly available reasoning and process-supervision benchmarks rather than collecting new user data. The datasets used in our experiments consist primarily of mathematical or reasoning problems, generated reasoning traces, and process labels. We did not identify personally identifying information or offensive content in the examples used for our experiments. We do not redistribute raw dataset contents beyond what is already publicly available under the original artifacts' access conditions.

\section{Information About Use of AI Assistants}
AI assistants were used to streamline code development and polish writing.

\end{document}